\documentclass[a4paper,11pt]{article}
\pdfoutput=1
\usepackage{a4wide}
\usepackage{amsmath,amsfonts,amssymb,amsthm}
\usepackage[colorlinks,citecolor=blue]{hyperref}
\usepackage{enumitem}
\usepackage{bbm}
\usepackage{natbib}
\usepackage{microtype}
\usepackage{graphicx}
\usepackage{extarrows}
\usepackage{pifont}
\usepackage{booktabs}

\setcitestyle{square}

\newtheorem{theorem}{Theorem}[section]
\newtheorem{lemma}[theorem]{Lemma}
\newtheorem{proposition}[theorem]{Proposition}
\newtheorem{corollary}[theorem]{Corollary}

\theoremstyle{definition}
\newtheorem{definition}[theorem]{Definition}

\newtheorem{remark}[theorem]{Remark}

\DeclareMathOperator*{\argmin}{argmin}

\numberwithin{equation}{section}
\numberwithin{table}{section}
\numberwithin{figure}{section}

\def \bE {\mathbb{E}}

\def \bN {\mathbb{N}}

\def \bR {\mathbb{R}}

\def \bZ {\mathbb{Z}}

\def \cA {\mathcal{A}}
\def \cB {\mathcal{B}}
\def \cC {\mathcal{C}}
\def \cD {\mathcal{D}}
\def \cE {\mathcal{E}}
\def \cF {\mathcal{F}}
\def \cG {\mathcal{G}}
\def \cH {\mathcal{H}}

\def \cJ {\mathcal{J}}
\def \cK {\mathcal{K}}
\def \cL {\mathcal{L}}
\def \cM {\mathcal{M}}
\def \cN {\mathcal{N}}
\def \cO {\mathcal{O}}
\def \cP {\mathcal{P}}

\def \cT {\mathcal{T}}

\def \cV {\mathcal{V}}
\def \cW {\mathcal{W}}
\def \cX {\mathcal{X}}

\def \NN {\mathcal{NN}}
\def \CNN {\mathcal{CNN}}
\def \MB {\mathcal{MB}}

\def \Ba {{\boldsymbol{a}}}
\def \Bb {{\boldsymbol{b}}}

\def \Be {{\boldsymbol{e}}}

\def \Bh {{\boldsymbol{h}}}
\def \Bi {{\boldsymbol{i}}}

\def \Bk {{\boldsymbol{k}}}

\def \Br {{\boldsymbol{r}}}
\def \Bs {{\boldsymbol{s}}}
\def \Bt {{\boldsymbol{t}}}
\def \Bu {{\boldsymbol{u}}}
\def \Bv {{\boldsymbol{v}}}
\def \Bw {{\boldsymbol{w}}}
\def \Bx {{\boldsymbol{x}}}
\def \By {{\boldsymbol{y}}}
\def \Bz {{\boldsymbol{z}}}

\def \Bll {{\boldsymbol{\ell}}}
\def \Bgamma {{\boldsymbol{\gamma}}}
\def \Bdelta {{\boldsymbol{\delta}}}

\def \Id {\,{\rm Id}\,}
\def \ind {\,{\rm ind}\,}

\def \sgn {\,{\rm sgn}\,}
\def \Pdim {\,{\rm Pdim}\,}
\def \Span {\,{\rm span}\,}
\def \deep {\,{\rm deep}\,}

\begin{document}
\title{Approximation and learning of anisotropic and mixed smooth functions by deep ReLU neural networks}
\author{
Yunfei Yang \thanks{School of Mathematics (Zhuhai) and Guangdong Province Key Laboratory of Computational Science, Sun Yat-sen University, Zhuhai, China. Corresponding author, E-mail: \href{mailto:yangyunfei@mail.sysu.edu.cn}{yangyunfei@mail.sysu.edu.cn}.}
\and
Jun Fan \thanks{Department of Mathematics, Hong Kong Baptist University, Kowloon, Hong Kong, China. E-mail: \href{mailto:junfan@hkbu.edu.hk}{junfan@hkbu.edu.hk}.}
}
\date{}
\maketitle

\begin{abstract}
This paper studies how efficiently deep ReLU neural networks can approximate and learn smooth functions. When the error is measured in $L^p([0,1]^d)$ norm and the approximator is a network with width $W$ and depth $L$, recent works have proven the supper approximation rate $\mathcal{O}((WL)^{-2s/d})$ for Besov space $\mathcal{B}^s_{q,r}([0,1]^d)$ under the Sobolev embedding condition $s/d>1/q-1/p$. In order to overcome the curse of dimensionality in this rate, we extent this result to anisotropic and mixed smooth function classes. We establish the approximation rate $\mathcal{O}((WL)^{-2\tilde{s}})$ for anisotropic Besov space $\mathcal{B}^{\boldsymbol{s}}_{q,r}([0,1]^d)$ with anisotropic smoothness $\boldsymbol{s}=(s_1,\dots,s_d)$ under the embedding condition $\tilde{s} > 1/q-1/p$, where the mean smoothness $\tilde{s} = (\sum_{i=1}^d s_i^{-1})^{-1}$. For mixed smooth Besov space $\mathcal{MB}^s_{q,r}([0,1]^d)$ with mixed smoothness $s>1/q-1/p$, we show that the approximation rate $\mathcal{O}((WL)^{-2s})$ holds up to logarithmic factors. Using these results, we also derive approximation bounds for the composition of anisotropic Besov functions. As an application, it is shown that deep ReLU neural networks can achieve minimax optimal rates up to logarithmic factors for a wide range of smooth function classes.

\end{abstract}

\section{Introduction}

Deep neural networks have achieved remarkable success in practical applications with high-dimensional
data, such as computer vision, speech recognition and natural language processing \citep{lecun2015deep,goodfellow2016deep}. This empirical success of deep learning suggests that deep neural networks can effectively solve high-dimensional machine learning problems due to their high flexibility and adaptability for representing and learning intrinsic structures of data. However, understanding which structural assumptions enable such practical performances of deep neural networks is still a central question in approximation and statistical learning theories.

The approximation capacity of neural networks has been extensively studied in the past decades. Early works showed that neural networks with one hidden layer are universal in the sense that they can approximate any continuous functions on any compact sets \citep{cybenko1989approximation,hornik1991approximation}. Due to the breakthrough of deep learning, approximation rates of deep neural networks on classical function classes such as H\"older, Sobolev and Besov spaces have been recently developed \citep{yarotsky2017error,yarotsky2018optimal,suzuki2019adaptivity,shen2020deep,shen2022optimal,lu2021deep,siegel2023optimal,yang2025optimal}. We remark that most recent works on approximation theory focus on the ReLU activation function, which is widely used in modern network architectures \citep{nair2010rectified}. These works gave estimations on the following $L^p$-approximation error of smooth function class $\cF$:
\begin{equation}\label{app error general}
\sup_{f\in \cF} \inf_{g\in \cG} \|f-g\|_{L^p([0,1]^d)},
\end{equation}
where $\cG$ is certain function class parameterized by deep ReLU neural networks. The seminal work of \citet{yarotsky2017error} showed that sparse network with $N$ parameters can achieve the rate $\cO(N/\log N)^{-s/d}$ for Sobolev space $\cF=\cW^{s,\infty}([0,1]^d)$. A series of papers by \citet{shen2020deep,shen2022optimal,lu2021deep} improved this result to the supper approximation rate $\cO((WL)^{-2s/d})$ up to logarithmic factors for fully-connected network $\cG=\NN(W,L)$ with width $W$ and depth $L$ (see (\ref{NN def}) below for precise definition). The recent works of \citet{siegel2023optimal} and \citet{yang2025optimal} extended the supper approximation rate to general Sobolev space $\cW^{s,q}([0,1]^d)$ and Besov space $\cB^s_{q,r}([0,1]^d)$ under the Sobolev embedding condition $s/d > 1/q-1/p$, which includes the nonlinear approximation regime $q<p$ \citep{devore1998nonlinear}. By using these approximation results, one can also show that deep neural networks can achieve the minimax optimal rate $\cO(n^{-2s/(2s+d)})$ up to logarithmic factors for learning Sobolev and Besov spaces with smoothness $s$ on $d$-dimensional space with $n$ training samples \citep{yang2025optimal}. Hence, we have a nearly complete picture for deep neural network approximation and learning theory of these function classes.

However, the approximation and learning rates for these function classes suffer from the curse of dimensionality, which cannot explain the practical success of deep learning on high-dimensional data. In order to overcome or weaken the curse of dimensionality, one often make certain structural assumptions on the data and target function, for example, the data distribution is low-dimensional \citep{nakada2020adaptive,chen2022nonparametric,jiao2023deep}, the target function has compositional structure \citep{schmidthieber2020nonparametric,kohler2021rate} or is anisotropic smooth \citep{suzuki2019adaptivity}, or is mixed smooth \citep{montanelli2019new,suzuki2021deep,yang2024optimal} or is of integral form such as Barron function \citep{barron1993universal,weinan2022barron,siegel2022sharp,parhi2023minimax,yang2025optimalshallow,yang2024nonparametric}. In this paper, we are mainly interested in the adaptivity of deep neural networks to anisotropic and mixed smoothness. For this direction, \citet{suzuki2019adaptivity,suzuki2021deep} showed that sparse network with $N$ parameters can achieve the approximation rate $\cO(N^{-\tilde{s}})$, where the mean smoothness $\tilde{s}$ is defined by (\ref{mean smoothness}) below, for anisotropic Besov space $\cB^\Bs_{q,r}([0,1]^d)$ with anisotropic smoothness $\Bs=(s_1,\dots,s_d)$, and the rate $\cO(N^{-s})$ up to logarithmic factors for mixed smooth Besov space $\MB^s_{q,r}([0,1]^d)$ (see Section \ref{sec: spaces} for the definition of these function spaces). These results generalized the analysis of \citet{yarotsky2017error} to anisotropic and mixed smooth functions. However, their analyses rely on the sparsity of the network and do not achieve the supper approximation rate. 

In this paper, we fill the above gap by establishing the supper approximation rate for general anisotropic and mixed smooth spaces. We give a framework to analyze the approximation capacity of deep neural networks by using adaptive multiscale anisotropic grid, which generalizes the idea of \citet{siegel2023optimal,yang2025optimal} from isotropic grid to the anisotropic case. Using this framework, we estimate the approximation error (\ref{app error general}) for a wide range of smooth function class $\cF$ when the approximator $\cG$ is parameterized by a fully-connected network $\NN(W,L)$ with width $W$ and depth $L$. We summarize our main approximation results as follows.
\begin{enumerate}[label=\textnormal{(\arabic*)}]
\item For anisotropic Besov space $\cB^\Bs_{q,r}([0,1]^d)$, Theorem \ref{app of anisotropic} establishes the supper approximation rate $\cO((WL)^{-2\tilde{s}})$ under the embedding condition $\tilde{s}>1/q-1/p$, where $\tilde{s}$ is mean smoothness defined by (\ref{mean smoothness}). Theorem \ref{anisotropic lower bound} shows that this rate is optimal up to logarithmic factors.

\item For mixed smooth Besov space $\MB^s_{q,r}([0,1]^d)$, Theorem \ref{app of mixed} establishes the supper approximation rate $\cO((WL)^{-2s})$ up to logarithmic factors under the embedding condition $s>1/q-1/p$. Theorem \ref{mixed lower bound} shows that this rate is optimal up to logarithmic factors.

\item For deep composition model of anisotropic Besov functions \citep{suzuki2021deep}, Theorem \ref{app of dcm} establishes the supper approximation rate, which is optimal up to logarithmic factors in some special cases as shown by Theorem \ref{dcm lower bound}. This result generalizes the analyses of \citet{schmidthieber2020nonparametric,kohler2021rate} for composition model of H\"older functions and improves the approximation rate of \citet{suzuki2021deep} for the same model.
\end{enumerate}
Using these approximation results, we are able to establish the learning rates of deep neural networks for the corresponding function classes. It is also known that these convergence rates are minimax optimal up to logarithmic factors in most cases.

The rest of this paper is organized as follows. We introduce anisotropic and mixed smooth Besov spaces in Section \ref{sec: spaces}. Our main approximation results for neural networks are given in Section \ref{sec: app}. In Section \ref{sec: learning}, we illustrate how to apply our approximation results to derive convergence rates for machine learning algorithms. The proofs of approximation and learning results are given in Sections \ref{sec: proofs of app} and \ref{sec: proofs of learning}, respectively.

We use the following notations throughout the paper. For any integers $i\le j$, we denote $[i:j]=\{i,i+1,\dots,j\}$. The maximum and minimum of $a,b\in \bR$ are denoted by $a \lor b$ and $a\land b$, respectively. Vectors are denoted by boldface letters. For vectors $\Ba=(a_1,\dots,a_d)$ and $\Bb=(b_1,\dots,b_d)$, the inequality $\Ba \le \Bb$ means $a_j\le b_j$ for all $j\in [1:d]$. For two quantities $A$ and $B$, $A \lesssim B$ (or $B \gtrsim A$) denotes the statement that $A\le CB$ for some constant $C>0$. We will also denote $A \asymp B$ when $A \lesssim B \lesssim A$. Finally, we remark that, unless otherwise specified, we will use $C$ to denote constants which may change from line to line. This convention is standard and convenient in analysis. The constants $C$ may depend on some other parameters and this dependence will be made clear from the context.

\section{Anisotropic and mixed smooth Besov spaces}\label{sec: spaces}

Let $\Omega \subseteq \bR^d$ be a bounded domain, which we will often take to be rectangles in the following. For $1\le q< \infty$, we denote by $L^q(\Omega)$ the set of functions $f$ whose $L^q$ norm on $\Omega$ is finite:
\[
\|f\|_{L^q (\Omega)}^q = \int_{\Omega} |f(\Bx)|^q d\Bx <\infty.
\]
When $q=\infty$, we have the modification $\|f\|_{L^\infty(\Omega)} = \sup_{\Bx\in \Omega} |f(\Bx)|$, since we mainly study continuous functions in the supremum norm.

For $k\in \bN_0$, $h\in \bR$ and $j\in [1:d]$, we define recursively the forward difference operator $\Delta_h^{k,j}$ of order $k$ with respect to the $j$-th coordinate by
\[
\Delta_h^{k,j} f(\Bx) := \Delta_h^{k-1,j} f(\Bx+h\Be_j) - \Delta_h^{k-1,j} f(\Bx), \quad \Delta_h^{0,j} f(\Bx) := f(\Bx),
\]
where $\Be_j$ is the $j$-th standard basis vector in $\bR^d$. It follows from the binomial theorem that 
\[
\Delta_h^{k,j} f(\Bx) = \sum_{i=0}^k (-1)^i \binom{k}{i} f(\Bx+ih\Be_j),
\]
which is well-defined for $f\in L^q(\Omega)$ on the set 
\[
\Omega(k,h,j):= \left\{\Bx\in \Omega: \Bx+kt\Be_j \in \Omega, \forall 0\le t\le h \right\}.
\] 
We note that $\Delta_{-h}^{k,j} f(\Bx) = (-1)^k \Delta_h^{k,j} f(\Bx-kh\Be_j)$ by definition. For $t\ge 0$, we define the $j$-th partial modulus of smoothness of order $k$ by 
\[
\omega_k^j(f,t,\Omega)_q := \sup_{0\le h\le t} \|\Delta_h^{k,j} f\|_{L^q(\Omega(k,h,j))}.
\]
It is well known that $\omega_k^j(f,t,\Omega)_q$ is non-decreasing on $t\in [0,\infty)$ and $\omega_k^j(f,\lambda t,\Omega)_q \le (\lambda+1)^k \omega_k^j(f,t,\Omega)_q$ for $\lambda>0$, see \citet[Section 2.7]{devore1993constructive} for instance. 

We define the anisotropic Besov spaces as follows. Let $\Bs = (s_1,\dots, s_d) \in (0,\infty)^d$ be an anisotropic smoothness vector. Let $1\le q,r\le \infty$ and $\Bk=(k_1,\dots,k_d)\in \bN^d$ with $k_j>s_j$ for $j=1,\dots,d$. We define the partial Besov semi-norm with respect to the $j$-th coordinate as 
\[
|f|_{\cB^{\Bs,j}_{q,r}(\Omega)} :=
\left( \int_0^\infty (t^{-s_j} \omega_{k_j}^j(f,t,\Omega)_q)^r \frac{dt}{t} \right)^{1/r},
\]
when $1\le r<\infty$ and 
\[
|f|_{\cB^{\Bs,j}_{q,\infty}(\Omega)} := \sup_{t>0} t^{-s_j} \omega_{k_j}^j(f,t,\Omega)_q.
\]
The anisotropic Besov space $\cB^\Bs_{q,r}(\Omega):= \{f\in L^q(\Omega): \|f\|_{\cB^\Bs_{q,r}(\Omega)}<\infty\}$ is defined through the norm
\[
\|f\|_{\cB^\Bs_{q,r}(\Omega)} := \|f\|_{L^q(\Omega)} + \sum_{j=1}^d |f|_{\cB^{\Bs,j}_{q,r}(\Omega)}.
\]
One can show that different choices of $k_j>s_j$ give equivalent norms by using standard estimates involving modulus of smoothness \citep{devore1993constructive}. Thus, we can simply choose $k_j=\lfloor s_j \rfloor +1$.

Next, we introduce the notion of mixed smoothness and the corresponding Besov spaces \citep{dung2018hyperbolic,triebel2019function}. For $\Bk=(k_1,\dots,k_d)\in \bN^d$, $\Bh=(h_1,\dots,h_d)\in \bR^d$ and $J \subseteq [1:d]$, the mixed difference operator $\Delta_\Bh^{\Bk,J}$ is defined through composition
\[
\Delta_\Bh^{\Bk,J} f(\Bx) = \left( \prod_{j\in J} \Delta_{h_j}^{k_j,j}\right) f(\Bx), \quad \Delta_\Bh^{\Bk,\emptyset} f(\Bx) := f(\Bx).
\]
One can check that the above definition is independent of the order of the composition and it is well-defined on the set 
\[
\Omega(\Bk,\Bh,J):= \left\{\Bx\in \Omega: (x_j+k_jt_j)_{j=1}^d \in \Omega, \forall \ t_j\in [0, h_j], j\in J \right\}.
\]
For $\Bt=(t_1,\dots,t_d)\in [0,\infty)^d$, the corresponding mixed modulus of smoothness is defined by
\[
\omega_\Bk^J(f,\Bt,\Omega)_q := \sup_{0\le h_j\le t_j, j\in J} \|\Delta_\Bh^{\Bk,J} f \|_{L^q(\Omega(\Bk,\Bh,J))}, \quad \omega_\Bk^\emptyset(f,\Bt,\Omega)_q := \|f\|_{L^q(\Omega)}.
\]
When $\Bk=(k,\dots,k)$ for some $k\in \bN$, we simplify the notations as $\Delta_\Bh^{k,J}$ and $\omega_k^J(f,\Bt,\Omega)_q$ for convenience. Note that $\omega_\Bk^J(f,\Bt,\Omega)_q$ is constant on $t_i$ for any $i\notin J$ and we have the equality
\[
\omega_\Bk^{\{j\}}(f,\Bt,\Omega)_q = \omega_{k_j}^j(f,t_j,\Omega)_q.
\]
Letting $\Bs = (s_1,\dots, s_d) \in (0,\infty)^d$ and $k_j=\lfloor s_j \rfloor +1$, we define the mixed smooth Besov semi-norm with respect to the coordinates in $J$ by 
\[
|f|_{\cB^{\Bs,J}_{q,r}(\Omega)} := 
\left( \int_{(0,\infty)^d} \left( \left(\prod_{j\in J}t_j^{-s_j}\right) \omega_\Bk^J(f,\Bt,\Omega)_q\right)^r \frac{d\Bt}{\prod_{j\in J}t_j} \right)^{1/r},
\]
when $1\le r<\infty$ and 
\[
|f|_{\cB^{\Bs,J}_{q,\infty}(\Omega)} := \sup_{\Bt\in (0,\infty)^d} \left(\prod_{j\in J}t_j^{-s_j}\right) \omega_\Bk^J(f,\Bt,\Omega)_q.
\]
For $J=\emptyset$, we use the modification $|f|_{\cB^{\Bs,\emptyset}_{q,r}(\Omega)} = \|f\|_{L^q(\Omega)}$. We can define the (anisotropic) mixed smooth Besov space $\MB^\Bs_{q,r}(\Omega):= \{f\in L^q(\Omega): \|f\|_{\MB^\Bs_{q,r}(\Omega)}<\infty\}$, which is also known as Besov space of dominating mixed smoothness, through the norm
\[
\|f\|_{\MB^\Bs_{q,r}(\Omega)} := \sum_{J\subseteq [1:d]} |f|_{\cB^{\Bs,J}_{q,r}(\Omega)}.
\]
Observing that the anisotropic Besov norm can also be expressed as
\[
\|f\|_{\cB^\Bs_{q,r}(\Omega)} = \sum_{J\subseteq [1:d], |J|\le 1} |f|_{\cB^{\Bs,J}_{q,r}(\Omega)},
\]
we have the natural embedding $\MB^\Bs_{q,r}(\Omega) \hookrightarrow \cB^\Bs_{q,r}(\Omega)$.

\begin{remark}\label{remark on functions}
Let $\Bs\in (0,\infty)^d$ and $1\le q,r\le \infty$. We have the following relations between smooth function classes \citep{devore1993constructive,dung2018hyperbolic,triebel2006theory,triebel2019function}.

(1) For $1\le r_1\le r_2\le \infty$, we have $\cB^\Bs_{q,r_1}(\Omega) \hookrightarrow \cB^\Bs_{q,r_2}(\Omega)$ and $\MB^\Bs_{q,r_1}(\Omega) \hookrightarrow \MB^\Bs_{q,r_2}(\Omega)$.

(2) When $\Bs=(s_0,\dots,s_0)$ for some $s_0>0$, $\cB^\Bs_{q,r}(\Omega) = \cB^{s_0}_{q,r}(\Omega)$ is the classical isotropic Besov space. These spaces are closely related to Sobolev spaces and H\"older spaces \citep{triebel2006theory}. For instance, when $s_0\notin \bN$, $\cB^{s_0}_{q,q}(\Omega) = \cW^{s_0,q}(\Omega)$ is the Sobolev-Slobodeckij space with the norm
\[
\|f\|_{\cW^{s_0,q}(\Omega)}^q := \|f\|_{L^q(\Omega)}^q + \sum_{|\boldsymbol{\gamma}|=\lfloor s_0 \rfloor} \left( \|\partial^{\boldsymbol{\gamma}} f\|_{L^q(\Omega)}^q + \int_{\Omega} \int_{\Omega} \frac{|\partial^{\boldsymbol{\gamma}} f(\Bx) - \partial^{\boldsymbol{\gamma}} f(\By)|^q}{|\Bx-\By|^{d+(s_0-\lfloor s_0 \rfloor)q}} d\Bx d\By \right),
\]
where $\gamma \in \bN_0^d$. For $q=\infty$, we make the standard modification and get the H\"older space $\cC^{s_0}(\Omega) = \cB^{s_0}_{\infty,\infty}(\Omega)$. When $s_0\in \bN$, the Sobolev space is equipped with the norm
\[
\| f\|_{\cW^{s_0,q}(\Omega)}^q := \|f\|_{L^q(\Omega)}^q + \sum_{|\boldsymbol{\gamma}|=s_0} \|\partial^{\boldsymbol{\gamma}} f\|_{L^q(\Omega)}^q,
\]
and we have the continuous embedding $\cB^{s_0}_{q,q}(\Omega) \hookrightarrow \cW^{s_0,q}(\Omega) \hookrightarrow \cB^{s_0}_{q,2}(\Omega)$ if $q\le 2$ and the reverse $\cB^{s_0}_{q,2}(\Omega) \hookrightarrow \cW^{s_0,q}(\Omega) \hookrightarrow \cB^{s_0}_{q,q}(\Omega)$ if $q\ge 2$. For $q=\infty$, $\cC^{s_0}(\Omega)=\cB^{s_0}_{\infty,\infty}(\Omega)$ is the Zygmund space, which is slightly larger than the Lipschitz space in the case $s_0=1$ \citep{devore1993constructive}.

(3) For $s\in \bN$ and $1\le q\le \infty$, the Korobov space \citep{bungartz2004sparse} is define as 
\[
\cK^s_q(\Omega) = \left\{ f\in L^q(\Omega): f|_{\partial \Omega} =0, \partial^{\boldsymbol{\gamma}} f\in L^q(\Omega), \forall \boldsymbol{\gamma}\in \bN_0^d, \mbox{with } \|\boldsymbol{\gamma}\|_\infty \le s \right\}.
\]
Note that the Sobolev space $\cW^{s,q}(\Omega)$ restricts the boundedness of $\|\partial^{\boldsymbol{\gamma}} f\|_{L^q(\Omega)}$ for all $|\boldsymbol{\gamma}|\le s$, while the Korobov space makes the same restriction for all $\|\boldsymbol{\gamma}\|_\infty \le s$. Hence, we have mixed derivatives $\partial^{\boldsymbol{\gamma}} f$ up to $\boldsymbol{\gamma} = (s,\dots,s)$ for Korobov functions. Since the mixed modulus of smoothness $\omega_{\boldsymbol{\gamma}}^J(f,\Bt,\Omega)_q \lesssim \prod_{j\in J} t_j^{\gamma_j} \|\partial^{\boldsymbol{\gamma}} f\|_{L^q(\Omega)}$, which can be established as \citet[Chapter 2, Eq. (7.12)]{devore1993constructive}, we have the embedding $\cK^s_q(\Omega) \hookrightarrow \MB^\Bs_{q,\infty}(\Omega)$ with smoothness $\Bs = (s,\dots,s)$.

(4) Many statistical models are included in the mixed smooth Besov spaces. For example, if $f_j\in \cB^{s_j}_{q,r}([0,1])$ for $j=1,\dots,d$, then
\[
f(\Bx) = \sum_{i=1}^d f_j(x_j) \in \MB^\Bs_{q,r}([0,1]^d),
\]
which is the well-known additive model \citep{meier2009high,raskutti2012minimax}. If $f_{j,k}\in \cB^{s_j}_{q,r}([0,1])$ for $j=1,\dots,d$ and $k=1,\dots,K$, then we have the tensor product model \citep{signoretto2013learning,kanagawa2016gaussian}
\[
f(\Bx) = \sum_{k=1}^K \prod_{j=1}^{d} f_{j,k}(x_j) \in \MB^\Bs_{q,r}([0,1]^d).
\]
\end{remark}

By part (1) of Remark \ref{remark on functions}, any approximation upper bounds for $\cB^\Bs_{q,\infty}(\Omega)$ and $\MB^\Bs_{q,\infty}(\Omega)$ can be applied to $\cB^\Bs_{q,r}(\Omega)$ and $\MB^\Bs_{q,r}(\Omega)$ for any $r\ge 1$, respectively. When $r=\infty$, we will omit the subscript $r$ and denote $\cB^\Bs_q(\Omega) := \cB^\Bs_{q,\infty}(\Omega)$ and $\MB^\Bs_q(\Omega) := \MB^\Bs_{q,\infty}(\Omega)$ for convenience. For simplicity, we will only consider mixed smooth Besov spaces of the form $\MB^{(s,\dots,s)}_{q,r}(\Omega)$ for some $s>0$ in this paper. With a slight abuse of notation, we will use the notation $\MB^s_{q,r}(\Omega)$ to indicate the space $\MB^{(s,\dots,s)}_{q,r}(\Omega)$. The above remark shows that these anisotropic and mixed smooth spaces contain a wide range of function classes in the literature.

\section{Approximation by deep neural networks}\label{sec: app}

Let us begin with an introduction of the neural network classes used in this paper. We will use the same terminology as \citet{yang2025optimal}. Given $L,N_1,\dots, N_L \in \bN$, we study the mapping $g:\bR^{d} \to \bR^{k}$ that can be parameterized by a fully-connected neural network with ReLU activation of the following form
\begin{equation}\label{NN def}
\begin{aligned}
g_0(\Bx) &= \Bx, \\
g_{\ell+1}(\Bx) &= \sigma(A_\ell g_\ell (\Bx) + b_\ell), \quad \ell = 0,1,\dots,L-1, \\
g(\Bx) &= A_L g_L(\Bx) +b_L,
\end{aligned}
\end{equation}
where $A_\ell \in \bR^{N_{\ell+1}\times N_{\ell}}$, $b_\ell\in \bR^{N_{\ell+1}}$ with $N_0 =d$ and $N_{L+1} =k$. The activation function $\sigma:\bR \to \bR$ is applied component-wisely and we focus on the case that $\sigma(t) = t_+ = \max\{t,0\}$ is the Rectified Linear Unit function (ReLU). Note that there is no activation function in the output layer, which is the usual convention in applications. The numbers $W:=\max\{N_1,\dots,N_L\}$ and $L$ are called the width and depth (the number of hidden layers) of the neural network, respectively. We are interested in the set of mappings that can be parameterized by ReLU neural networks with width $W$ and depth $L$. We denote this function class by $\cN\cN_{d,k}(W,L)$. Since the input dimension $d$ and the output dimension $k$ are often clear from contexts, we will often simply denote the network by $\cN\cN(W,L)$ for convenience. 

In this section, we estimate how well the neural network class $\cN\cN(W,L)$ approximates smooth function classes:
\begin{equation}\label{app error}
\sup_{f\in \cF} \inf_{g\in \NN(W,L)} \|f-g\|_{L^p([0,1]^d)},
\end{equation}
where $\cF$ is the anisotropic Besov space, or the mixed smooth Besov space, or the composition of Besov functions, which is called the deep composition model \citep{suzuki2021deep}. We will also show that our estimates are nearly tight in many cases. The proofs in this section are deferred to Section \ref{sec: proofs of app}.

\subsection{Anisotropic Besov space}

The complexity of the anisotropic Besov space $\cB^\Bs_q([0,1]^d)$ is characterized by the harmonic mean of the smoothness $\Bs$. Specifically, we define $\tilde{s}>0$ for any $\Bs = (s_1,\dots, s_d)$ by
\begin{equation}\label{mean smoothness}
\frac{1}{\tilde{s}} = \sum_{j=1}^d \frac{1}{s_j}.
\end{equation}
Our first main result determines the optimal rates of the approximation error (\ref{app error}) when the target function class $\cF$ is the unit ball of the anisotropic Besov space $\cB^\Bs_q([0,1]^d)$ under the condition $\tilde{s}>1/q-1/p$, which guarantees that $\cB^\Bs_q([0,1]^d)$ is embedded in $L^p([0,1]^d)$.

\begin{theorem}\label{app of anisotropic}
Let $\Bs\in (0,\infty)^d$ and $1\le p,q\le \infty$. If $\tilde{s}>1/q-1/p$, then for sufficiently large $W,L\in \bN$,
\[
\inf_{g\in \NN(W,L)} \|f-g\|_{L^p([0,1]^d)} \le C \|f\|_{\cB^\Bs_q([0,1]^d)} (WL)^{-2\tilde{s}},
\]
for some constant $C$ depending on $\Bs,p,q$ and $d$.
\end{theorem}

The recent work of \citet{suzuki2021deep} also studied neural network approximation of anisotropic Besov functions. They established the approximation rate $\cO(N^{-\tilde{s}})$ for sparse neural networks, where $N$ is the number of nonzero parameters in the network. Note that, for fully-connected networks in our setting, the number of parameters in the network is $N=\cO(W^2L)$ when $L\ge 2$. Hence, Theorem \ref{app of anisotropic} gives the approximation rate $\cO((WL)^{-2\tilde{s}}) \le \cO(N^{-\tilde{s}})$ so that we can recover the result of \citet{suzuki2021deep} and show that sparsity is not necessary to obtain this rate. Furthermore, we give a precise trade-off between width and depth. In particular, when the width $W$ is bounded, our result implies that the approximation rate can be improved to $\cO(L^{-2\tilde{s}}) \le \cO(N^{-2\tilde{s}})$. 

For the isotropic Besov space $\cB^s_q([0,1]^d)$, which corresponds to $\Bs=(s,\dots,s)$, we have $\tilde{s} = s/d$. Thus, Theorem \ref{app of anisotropic} shows that, if $s/d>1/q-1/p$, then 
\[
\inf_{g\in \NN(W,L)} \|f-g\|_{L^p([0,1]^d)} \le C \|f\|_{\cB^s_q([0,1]^d)} (WL)^{-2s/d},
\]
which is exactly the result of \citet{yang2025optimal}. Using this upper bound, we are able to recover almost all approximation results for Besov functions (hence for H\"older and Sobolev functions) in the literature \citep{yarotsky2017error,yarotsky2018optimal,yarotsky2020phase,shen2020deep,shen2022optimal,lu2021deep,suzuki2019adaptivity,suzuki2021deep,siegel2023optimal}. A detailed comparison is given in Table \ref{table for anisotropic}. We remark that, for the isotropic case, the approximation rate suffers from the curse of dimensionality, which is an undesirable property for high dimensional machine learning problems. Our result shows that this can be overcome in the anisotropic case when $\tilde{s} = ( \sum_{j=1}^d 1/s_j)^{-1}$ is independent of the dimension $d$, i.e., when we have high smoothness $s_j$ in most directions. (However, we emphasis that the constant $C$ in Theorem \ref{app of anisotropic} may still depend exponentially on the dimension.)

\begin{table}[htbp]
\setlength{\tabcolsep}{3pt}
\begin{tabular}{c c c c} \toprule
 & \textbf{approximation error} & \textbf{anisotropic} & \textbf{condition} \\ \midrule
\citet{yarotsky2017error} & $(N/\log N)^{-s/d}$ & \ding{55}  & $p\le q=\infty, s\in \bN$  \\ \midrule
\citet{suzuki2019adaptivity} & $N^{-s/d}$ & \ding{55} & $s/d>1/q - 1/p$ \\ \midrule 
\citet{yarotsky2020phase} & $(L/\log L)^{-2s/d}$ & \ding{55} & $p\le q=\infty$ \\ \midrule 
\citet{lu2021deep} & $(WL/(\log W \log L))^{-2s/d}$ & \ding{55} & $p\le q=\infty$ \\ \midrule 
\citet{shen2022approximation} & $(WL {\sqrt{\log W}})^{-2s/d}$ & \ding{55} & $p\le q=\infty, s\le 1$ \\ \midrule 
\citet{siegel2023optimal} & $L^{-2s/d}$ & \ding{55} & $s/d>1/q - 1/p$ \\ \midrule 
\citet{yang2025optimal} & $(WL)^{-2s/d}$ & \ding{55} & $s/d>1/q - 1/p$ \\ \midrule 
\citet{suzuki2021deep} & $N^{-\tilde{s}}$ & \ding{51} & $\tilde{s}>1/q - 1/p$ \\ \midrule
Theorem \ref{app of anisotropic} & $(WL)^{-2\tilde{s}}$ & \ding{51} & $\tilde{s}>1/q - 1/p$ \\ \bottomrule 
\end{tabular}
\caption{A comparison of approximation upper bounds for the Besov space $\cB^\Bs_q([0,1]^d)$. The error is measured in $L^p([0,1]^d)$ and is estimated by the number of nonzero parameters $N$, width $W$ and depth $L$. When $\Bs=(s,\dots,s)$ is isotropic, $\tilde{s}=s/d$.}
\label{table for anisotropic}
\end{table}

We note that, for isotropic spaces with smoothness $s\le 1$, \citet{shen2022approximation} gave the upper bound $(WL {\sqrt{\log W}})^{-2s/d}$, which is slightly better than Theorem \ref{app of anisotropic} by a $(\log W)^{-s/d}$ factor. Thus, it might be possible to improve our result. However, this improvement is minor, since \citet{siegel2023optimal} showed the existence of $f$ in the isotropic Besov space with $\|f\|_{\cB^s_{q,r}([0,1]^d)} \le 1$ such that 
\[
\inf_{g\in \NN(W,L)} \|f-g\|_{L^p([0,1]^d)} \ge C (W^2L^2\log(WL) \land W^3L^2)^{-s/d}.
\]
for some constant $C>0$ depending on $s,p$ and $d$. Note that there are two lower bounds in the above inequality due to different estimations of the pseudo-dimension of ReLU networks (see \citet{bartlett2019nearly} and inequality (\ref{Pdim of NN}) below). We generalize these lower bounds to the anisotropic case in the next theorem.

\begin{theorem}\label{anisotropic lower bound}
Let $\Bs\in (0,\infty)^d$ and $1\le p,q,r\le \infty$. Then there exits $f$ with $\|f\|_{\cB^\Bs_{q,r}([0,1]^d)} \le 1$ such that
\[
\inf_{g\in \NN(W,L)} \|f-g\|_{L^p([0,1]^d)} \ge C (W^2 L^2\log (WL) \land W^3L^2)^{-\tilde{s}},
\]
for $WL\ge 2$ and some constant $C>0$ depending on $\Bs,p$ and $d$.
\end{theorem}

Theorem \ref{anisotropic lower bound} shows that the upper bound in Theorem \ref{app of anisotropic} is optimal up to a $(\log (WL))^{-\tilde{s}}$ factor. Furthermore, when the width $W$ is bounded and sufficiently large, we actually get the tight estimate
\[
\sup_{\|f\|_{\cB^\Bs_{q,r}([0,1]^d)} \le 1} \inf_{g\in \NN(W,L)} \|f-g\|_{L^p([0,1]^d)} \asymp L^{-2\tilde{s}},\quad \mbox{if } \tilde{s}>1/q-1/p.
\]

\subsection{Mixed smooth Besov space}

We have shown that one can overcome the curse of dimensionality on the approximation rate for anisotropic smooth functions when the smoothness index $\tilde{s}$ is independent of the dimension. Another well known way to weaken the curse of dimensionality is to consider the mixed smoothness. (Near) dimension independent approximation rates of neural networks for mixed smooth functions have been established by \citet{montanelli2019new,blanchard2022shallow,suzuki2019adaptivity,yang2024optimal}. We generalize and improve these results in the next theorem.

\begin{theorem}\label{app of mixed}
Let $s>0$ and $1\le p,q\le \infty$. If $s>1/q-1/p$, then for sufficiently large $W,L\in \bN$,
\[
\inf_{g\in \NN(W,L)} \|f-g\|_{L^p([0,1]^d)} \le C \|f\|_{\MB^s_q([0,1]^d)} (WL)^{-2s}(\log W \log L)^{(d-1)(2s+1)},
\]
for some constant $C$ depending on $s,p,q$ and $d$.
\end{theorem}

Let us compare Theorem \ref{app of mixed} with existing works. \citet{suzuki2019adaptivity} showed that one can use a ReLU network with $N$ nonzero parameters to approximate mixed smooth functions in $\MB^s_q([0,1]^d)$ with error $\widetilde{\cO}(N^{-s})$, where we use $\widetilde{\cO}$ to hide logarithmic factors since this already reveals the advantages of our result. For fully-connected networks in our setting, the number of parameters is $N=\cO(W^2L)$ so that we get the same approximation rate $\widetilde{\cO}((WL)^{-2s}) \le \widetilde{\cO}(N^{-s})$ in general. However, when the width $W$ is bounded, Theorem \ref{app of mixed} improves the result of \citet{suzuki2019adaptivity} to the supper approximation rate $\widetilde{\cO}(L^{-2s}) \le \widetilde{\cO}(N^{-2s})$. For the Korobov space $\cK^2_\infty([0,1]^d)$ of order two, which can be embedded in $\MB^2_\infty([0,1]^d)$, \citet{montanelli2019new} proved the rate $\widetilde{\cO}(N^{-2})$ for sparse networks with $N$ parameters. The recent work of \citet{yang2024optimal} improved this result to the supper approximation rate $\cO((WL)^{-4}(\log W \log L)^{5d+3})$. For comparison, Theorem \ref{app of mixed} gives the rate $\cO((WL)^{-4}(\log W \log L)^{5d-5})$, which is slightly better. There is another line of works studying the approximation capacity of convolutional neural networks (CNNs) for Korobov functions \citep{mao2022approximation,li2025higher}. We show in the Appendix \ref{appendix: cnn} that our results can also be used to derive approximation rates for CNNs by using the idea from \citet{zhou2020theory,zhou2020universality}. A detailed comparison is given in Table \ref{table for mixed}.

\begin{table}[htbp]
\centering
\setlength{\tabcolsep}{4pt}
\begin{tabular}{c c c c} \toprule
 & \textbf{target function} & \textbf{model} & \textbf{approximation error} \\ \midrule
\citet{montanelli2019new} & $\cK^2_\infty$ & sparse NN  & $\widetilde{\cO}(N^{-2})$  \\ \midrule
\citet{yang2024optimal} & $\cK^2_\infty$ & fully-connected NN  & $\widetilde{\cO}((WL)^{-4})$  \\ \midrule
\citet{mao2022approximation} & $\cK^2_p$ & CNN  & $\widetilde{\cO}(L^{-2+1/p})$  \\ \midrule
\citet{li2025higher} & $\cK^s_p$, $s\in \bN, s\ge 2$ & CNN  & $\widetilde{\cO}(L^{-s})$  \\ \midrule
\citet{suzuki2019adaptivity} & $\MB^s_q$, $s>\frac{1}{q}-\frac{1}{p}$ & sparse NN  & $\widetilde{\cO}(N^{-s})$  \\ \midrule
Theorem \ref{app of mixed} & $\MB^s_q$, $s>\frac{1}{q}-\frac{1}{p}$ & fully-connected NN & $\widetilde{\cO}((WL)^{-2s})$ \\ \midrule 
Appendix \ref{appendix: cnn} & $\MB^s_q$, $s>\frac{1}{q}-\frac{1}{p}$ & CNN & $\widetilde{\cO}(L^{-2s})$  \\ \bottomrule 
\end{tabular}
\caption{A comparison of approximation upper bounds for mixed smooth functions. The error is measured in $L^p([0,1]^d)$ and is estimated by the number of nonzero parameters $N$, width $W$ and depth $L$.}
\label{table for mixed}
\end{table}

Similar to Theorem \ref{anisotropic lower bound} for anisotropic Besov spaces, we give approximation lower bounds for mixed smooth Besov spaces in the next theorem. 

\begin{theorem}\label{mixed lower bound}
Let $s>0$ and $1\le p,q,r\le \infty$. If $s>1/r-1/2$, then there exits $f$ with $\|f\|_{\MB^s_{q,r}([0,1]^d)} \le 1$ such that
\[
\inf_{g\in \NN(W,L)} \|f-g\|_{L^p([0,1]^d)} \ge C (W^2 L^2\log (WL) \land W^3L^2)^{-s} (\log (WL))^{(d-1)(s+1/2-1/r)-s},
\]
for $WL\ge 2$ and some constant $C$ depending on $s,p,r$ and $d$.
\end{theorem}

Theorem \ref{mixed lower bound} shows that the upper bound in Theorem \ref{app of anisotropic} is optimal up to logarithmic factors. Note that, in contrast to the anisotropic spaces, the lower bounds for the mixed Besov space $\MB^s_{q,r}$ depend on the third index $r$. This is due to the fact that the metric entropy of $\MB^s_{q,r}$ depends on $r$, as pointed out by \citet{vybiral2006function,dung2018hyperbolic}. Thus, it might be possible to improve the upper bound in Theorem \ref{app of mixed} by taking into account the effect of the index $r$ and we leave this for future study.

\subsection{Deep composition model}

Due to the compositional structure of neural networks, it is natural to model unknown target function by a composition of nonlinear functions \citep{schmidthieber2020nonparametric,suzuki2021deep,kohler2021rate}. In this section, we study neural network approximation of the composition of anisotropic Besov functions as in \citet{suzuki2021deep}. Note that one can also consider the composition of mixed smooth Besov functions or a combination of anisotropic and mixed smooth Besov functions. Since the analyses for these function classes are similar, we do not pursue this for simplicity.

Let $d_1=d$ and $d_{M+1}=1$ be the input and output dimensions respectively. For $m\in [1:M]$, let $d_m$ be the dimension of the $m$-th layer, and let $\Bs_m\in (0,\infty)^{d_m}$ denote the smoothness parameter in the $m$-th layer. We define the deep composition model with $M$ layers as 
\[
\cF_{\deep} := \left\{ f_M\circ \cdots \circ f_1\ |\ f_m: [0,1]^{d_m} \to [0,1]^{d_{m+1}}, \|f_{m,j}\|_{\cB^{\Bs_m}_{q_m}([0,1]^{d_m})}\le 1, \forall j\in [1:d_{m+1}] \right\}.
\]
Note that the cube $[0,1]^{d_m}$ can be replaced by any rectangles in $\bR^{d_m}$ by scaling. One can also relax the norm assumption $\|f_{m,j}\|_{\cB^{\Bs_m}_{q_m}([0,1]^{d_m})}\le 1$. This model includes many models studied in the literature. For example, if the input is distributed on a low-dimensional manifold and the target function $f$ is anisotropic along a coordinate direction on the manifold, then we can write $f(\Bx) = h(\varphi(\Bx))$, where $\varphi: \bR^d \to \bR^{\tilde{d}}$ maps $\Bx$ to its coordinate on the $\tilde{d}$-dimensional manifold and $h$ is an anisotropic Besov function. We refer the reader to \citet{suzuki2021deep} for more discussions on this model.

In order to describe our approximation result, we need to introduce the smoothness index of the deep composition model. For the anisotropic smoothness vector $\Bs_m =(s_{m,1},\dots,s_{m,d_m})$, we define its mean smoothness $\tilde{s}_m$ by (\ref{mean smoothness}) and denote $\underline{s}_m = \min_{1\le j\le d_m} s_{m,j}$. Notice that, when $\tilde{s}_m >1/q_m$, we can embed the anisotropic Besov space into the H\"older-Zygmund space as $\cB^{\Bs_m}_{q_m}([0,1]^d) \hookrightarrow \cC^{\underline{s}_m-\underline{s}_m/(\tilde{s}_m q_m)}([0,1]^d)$, see \citet[Proposition 1]{suzuki2021deep} for instance. For any small $\epsilon>0$, we define
\begin{equation}\label{gamma_m}
\gamma_m := 
\begin{cases}
\underline{s}_m-\underline{s}_m/(\tilde{s}_m q_m), & \mbox{if }\underline{s}_m-\underline{s}_m/(\tilde{s}_m q_m)<1, \\
1-\epsilon, & \mbox{if }\underline{s}_m-\underline{s}_m/(\tilde{s}_m q_m)=1, \\
1, & \mbox{if }\underline{s}_m-\underline{s}_m/(\tilde{s}_m q_m)>1.
\end{cases}
\end{equation}
Then, the function $f_{m,j}$ in the $m$-th layer are $\gamma_m$-H\"older continuous. Note that functions in the Zygmund space $\cC^1$ may not be Lipschitz continuous, so we use the $(1-\epsilon)$-H\"older continuity in this case. The H\"older smoothness will be used to estimate the influence of the approximation error in the internal layers to the entire function, which can be easily understood from the proof of Theorem \ref{app of dcm} given in Section \ref{sec: app of anisotropic} below. We define the smoothness index for the deep composition model by
\begin{equation}\label{smoothness of dcm}
s^* := \min_{1\le m\le M} \tilde{s}_m \prod_{k=m+1}^{M} \gamma_k.
\end{equation}
In the next theorem, we characterize the approximation error of the deep composition model by this smoothness index.

\begin{theorem}\label{app of dcm}
If $\tilde{s}_m >1/q_m$ for $m\in [1:M]$, then for sufficiently large $W,L\in \bN$,
\[
\sup_{f\in \cF_{\deep}} \inf_{g\in \NN(W,L)} \|f-g\|_{L^\infty([0,1]^d)} \le C (WL)^{-2s^*},
\]
for some constant $C$ depending on $(\Bs_m,q_m,d_m)_{m=1}^M$.
\end{theorem}

Similar approximation bound has been established for sparse neural networks in \citet{schmidthieber2020nonparametric,suzuki2021deep,kohler2021rate}. The main difference is that \citet{schmidthieber2020nonparametric,suzuki2021deep} estimated the approximation error by the number of nonzero parameters, while \citet{kohler2021rate} gave a precise trade-off between width and depth and derived the supper approximation rate for deep composition model of H\"older functions. Theorem \ref{app of dcm} generalizes their results to the composition of anisotropic Besov functions.

Next, we consider the approximation lower bound of the deep composition model. For the lower bound, we are only able to prove a convergence rate that is slightly slower than the rate in Theorem \ref{app of dcm}. Instead of the smoothness index (\ref{smoothness of dcm}), we use the following index
\begin{equation}\label{smoothness of dcm lower}
s^{**} := \min_{1\le m\le M} \tilde{s}_m \prod_{k=m+1}^{M} \widetilde{\gamma}_k, \quad \widetilde{\gamma}_k:= (\underline{s}_k-1/q_k) \land 1.
\end{equation}
Observe that, when all $q_m=\infty$, $s^* \le s^{**}< s^*+\epsilon$ for any $\epsilon>0$. Hence, the following theorem implies that the upper bound in Theorem \ref{app of dcm} is nearly optimal in this case.

\begin{theorem}\label{dcm lower bound}
If $1\le p\le \infty$, $\underline{s}_m >1/q_m$ for $m\in [1:M]$ and $d_1\ge d_2 \ge \dots \ge d_M$, then there exists $f\in \cF_{\deep}$ such that 
\[
\inf_{g\in \NN(W,L)} \|f-g\|_{L^p([0,1]^d)} \ge C (W^2 L^2\log (WL) \land W^3L^2)^{-s^{**}},
\]
for $WL\ge 2$ and some constant $C$ depending on $(\Bs_m,q_m,d_m)_{m=1}^M$.
\end{theorem}

\section{Nonparametric regression}\label{sec: learning}

In this section, we illustrate how to apply our approximation results to derive convergence rates for machine learning algorithms. We are going to derive new learning rates for the least squares estimator in the classical nonparametric regression setting, which is a hot topic in recent research \citep{schmidthieber2020nonparametric,suzuki2019adaptivity,nakada2020adaptive,kohler2021rate,chen2022nonparametric,jiao2023deep}. Although we only consider regression problem in this paper, we remark that similar analysis can be applied to study other learning problems, such as classification \citep{kim2021fast,yang2025rates} and distribution learning by diffusion modeling \citep{oko2023diffusion}.

Suppose we have a dataset of $n\ge 2$ samples $\cD_n = \{(X_i,Y_i)\}_{i=1}^n$, which are independent and identically distributed as a $\bR^d \times \bR$-valued random vector $(X,Y)$. We use $\mu$ to denote the marginal distribution of the covariate $X$. Throughout this section, we assume that $\mu$ is supported on $[0,1]^d$ and there exists a constant $c>0$ such that
\begin{equation}\label{noise condition}
\bE\left[e^{cY^2}\right]<\infty.
\end{equation}
Note that this condition is weaker than the sub-Gaussian condition by \citet[Theorem 2.6 (IV)]{wainwright2019high}. In nonparametric regression, our goal is to estimate the so-called regression function $f(x) = \bE[Y|X=x]$ from the observed data $\cD_n$ without assuming any parametric form of $f$. One of the most popular estimators is the least squares
\begin{equation}\label{least squares}
\widehat{h}_n \in \argmin_{h\in \cH} \frac{1}{n} \sum_{i=1}^n (h(X_i)- Y_i)^2,
\end{equation}
where $\cH$ is a suitably chosen hypothesis class. For simplicity, we assume here and in the sequel that the minimum above indeed exists. In deep learning, the function class $\cH$ is parameterized by deep neural networks. Here, we study the case that $\cH =\NN(W,L)$, where the width $W$ and depth $L$ should depend on the sample size $n$ so that we can obtain convergence rate for the estimator $\widehat{h}_n$. The performance of the estimation is measured by the expected risk
\[
\cL(\widehat{h}_n) := \bE_{(X,Y)} [(\widehat{h}_n(X)-Y)^2],
\]
which is equivalent to evaluating the estimator by the excess risk
\[
\|\widehat{h}_n - f\|_{L^2(\mu)}^2 = \cL(\widehat{h}_n) - \cL(f).
\]

In the statistical analysis of learning algorithms, we often require that the hypothesis class is uniformly bounded. To achieve this, we define the truncation operator $\cT_B$ with level $B>0$ for real-valued functions $h$ as
\begin{equation}\label{trancation}
\cT_Bh(\Bx) := 
\begin{cases}
h(\Bx) &\quad \mbox{if }|h(\Bx)|\le B, \\
\sgn(h(\Bx)) B &\quad \mbox{if } |h(\Bx)|> B.
\end{cases}
\end{equation}
For a function class $\cH$ containing real-valued functions, we denote $\cT_B \cH := \{\cT_Bh: h\in \cH\}$ for convenience. In the following, we give convergence rates for the truncated least squares estimator $\cT_{B_n} \widehat{h}_n$, where $B_n = c\log n$ for some constant $c>0$. Note that the hypothesis class in the optimization problem (\ref{least squares}) has not been truncated and we only apply truncation when we evaluate the estimator since truncation does not increase the excess risk.

In order to get any learning rate, it is necessary to make some prior assumption on the regression function $f$. Typically, one assume that $f\in \cF$ for certain smooth function class $\cF$. We will analyze the case that $\cF$ is the anisotropic Besov space, or the mixed smooth Besov space, or the deep composition model so that our approximation results can be applied. All proofs in this section are given in Section \ref{sec: proofs of learning}.

For anisotropic Besov spaces, we prove the following convergence result.

\begin{theorem}\label{learn anisotropic}
Suppose $\Bs\in (0,\infty)^d$ and $1\le q\le \infty$ satisfy $\tilde{s}>1/q$. Assume that the data distribution satisfies (\ref{noise condition}) and the regression function $f\in \cB^\Bs_q([0,1]^d)$ with $\|f\|_{\cB^\Bs_q([0,1]^d)}\le 1$ and $\|f\|_{L^\infty([0,1]^d)} \le B$ for some constant $B\ge 1$. Let $\widehat{h}_n$ be the least squares estimator (\ref{least squares}) with $\cH=\NN(W_n,L_n)$, where $W_n$ and $L_n$ are sufficiently large so that Theorem \ref{app of anisotropic} can be applied with $p=\infty$. If $B_n = c\log n$ for some constant $c>0$ and 
\[
W_nL_n \asymp n^{\frac{1}{4\tilde{s}+2}} (\log n)^{-\frac{2}{2\tilde{s}+1}},
\]
then we have
\[
\bE_{\cD_n} \|\cT_{B_n}\widehat{h}_n-f\|_{L^2(\mu)}^2 \le C n^{-\frac{2\tilde{s}}{2\tilde{s}+1}} (\log n)^{\frac{8\tilde{s}}{2\tilde{s}+1}},
\]
where  $\bE_{\cD_n}$ indicates the expectation with respect to the training data $\cD_n$ and $C$ is a constant independent of the regression function $f$ and the sample size $n$.
\end{theorem}

\begin{remark}
If $\mu$ is absolutely continuous with respect to the Lebesgue measure with density $p_X$ which satisfies $0\le p_X(x)\le C<\infty$ on $[0,1]^d$, then we can relax the condition $\tilde{s}>1/q$ to $\tilde{s}>1/q-1/2$. Because, in this case, for any $f,g\in L^2([0,1]^d)$,
\[
\|f-g\|_{L^2(\mu)} \le C\|f-g\|_{L^2([0,1]^d)}.
\]
Hence, we can apply Theorem \ref{app of anisotropic} with $p=2$ rather than $p=\infty$ to estimate the approximation error in $L^2(\mu)$. We have similar remark for Theorem \ref{learn mixed} below.
\end{remark}

It is well-known that $n^{-2s/(2s+d)}$ is the minimax optimal rate for learning isotropic Besov space with smoothness $s$ \citep{donoho1998minimax,stone1982optimal}. \citet{suzuki2021deep} established similar minimax rate for anisotropic Besov spaces. Specifically, they showed that, when $\tilde{s}>1/q-1/2$ and $\mu$ is the uniform distribution,
\[
\inf_{\widehat{h}_n} \sup_{\|f\|_{\cB^\Bs_q([0,1]^d)}\le 1} \bE_{\cD_n} \|f- \widehat{h}_n\|_{L^2(\mu)}^2 \gtrsim n^{-\frac{2\tilde{s}}{2\tilde{s}+1}},
\]
where the infimum is taken over all measurable functions $\widehat{h}_n$ of the dataset $\cD_n$. Consequently, the convergence rate in Theorem \ref{learn anisotropic} is minimax optimal up to logarithmic factors. Note that similar results have been established in \citet{schmidthieber2020nonparametric} and \citet{suzuki2021deep} for isotropic H\"older spaces and anisotropic Besov spaces respectively. However, their results rely on the sparsity of neural networks and hence one need to optimize over different network architectures to obtain the optimal rates, which is hard to implement due to the unknown locations of the non-zero parameters. \citet{kohler2021rate} and \citet{yang2025optimal} showed that fully-connected networks can also achieve the minimax rates for isotropic smooth function classes. Theorem \ref{learn anisotropic} generalizes their results to the anisotropic setting.

Similar to Theorem \ref{learn anisotropic} for anisotropic functions, we can also derive convergence rates for learning mixed smooth functions by applying Theorem \ref{app of mixed}.

\begin{theorem}\label{learn mixed}
Suppose $1\le q\le \infty$ and $s>1/q$. Assume that the data distribution satisfies (\ref{noise condition}) and the regression function $f\in \MB^s_q([0,1]^d)$ with $\|f\|_{\MB^s_q([0,1]^d)}\le 1$ and $\|f\|_{L^\infty([0,1]^d)} \le B$ for some constant $B\ge 1$. Let $\widehat{h}_n$ be the least squares estimator (\ref{least squares}) with $\cH=\NN(W_n,L_n)$, where $W_n$ and $L_n$ are sufficiently large so that Theorem \ref{app of anisotropic} can be applied with $p=\infty$. If $B_n = c\log n$ for some constant $c>0$ and 
\[
W_nL_n \asymp n^{\frac{1}{4s+2}} (\log n)^{-\frac{2}{2s+1}+2d-2},
\]
then we have
\[
\bE_{\cD_n} \|\cT_{B_n}\widehat{h}_n-f\|_{L^2(\mu)}^2 \le C n^{-\frac{2s}{2s+1}} (\log n)^{\frac{8s}{2s+1}+4d-4},
\]
where $C$ is a constant independent of the regression function $f$ and the sample size $n$.
\end{theorem}

\citet{suzuki2019adaptivity} proved the following lower bound for learning mixed smooth Besov class $\MB^s_{q,r}$ when $s>1/q-1/2$ and $\mu$ is the uniform distribution,
\[
\inf_{\widehat{h}_n} \sup_{\|f\|_{\MB^s_{q,r}([0,1]^d)}\le 1} \bE_{\cD_n} \|f- \widehat{h}_n\|_{L^2(\mu)}^2 \gtrsim n^{-\frac{2s}{2s+1}} (\log n)^{\frac{2(d-1)(s+1/2-1/r)_+}{2s+1}}.
\]
They also showed that sparse networks can achieve the minimax rate up to logarithmic factors. Theorem \ref{learn mixed} extends this result to fully-connected network and gives a precise trade-off between width and depth to obtain the minimax rate.

For the deep composition model $\cF_{\deep}$, the convergence rate is given in the following theorem.

\begin{theorem}\label{learn dcm}
Suppose $\Bs_m\in (0,\infty)^{d_m}$ satisfies $\tilde{s}_m >1/q_m$ for $m\in [1:M]$. Assume that the data distribution satisfies (\ref{noise condition}) and the regression function $f\in \cF_{\deep}$ with $\|f\|_{L^\infty([0,1]^d)} \le B$ for some constant $B\ge 1$. Let $\widehat{h}_n$ be the least squares estimator (\ref{least squares}) with $\cH=\NN(W_n,L_n)$, where $W_n$ and $L_n$ are sufficiently large so that Theorem \ref{app of anisotropic} can be applied with $p=\infty$. If $B_n = c\log n$ for some constant $c>0$ and 
\[
W_nL_n \asymp n^{\frac{1}{4s^*+2}} (\log n)^{-\frac{2}{2s^*+1}},
\]
then we have
\[
\bE_{\cD_n} \|\cT_{B_n}\widehat{h}_n-f\|_{L^2(\mu)}^2 \le C n^{-\frac{2s^*}{2s^*+1}} (\log n)^{\frac{8s^*}{2s^*+1}},
\]
where $C$ is a constant independent of the regression function $f$ and the sample size $n$.
\end{theorem}

The learning rate of the deep composition model was first studied by \citet{schmidthieber2020nonparametric}, who established the minimax rate (up to logarithmic factors) for learning the composition of H\"older functions by sparse networks. \citet{kohler2021rate} showed that the same learning rate also holds for fully-connected networks. \citet{suzuki2021deep} extended the analysis of \citet{schmidthieber2020nonparametric} to learning the composition of anisotropic Besov functions using sparse networks, and derived the following lower bound for the convergence rate
\[
\inf_{\widehat{h}_n} \sup_{f\in \cF_{\deep}} \bE_{\cD_n} \|f- \widehat{h}_n\|_{L^2(\mu)}^2 \gtrsim n^{-\frac{2s^{**}}{2s^{**}+1}},
\]
where $s^{**}$ is defined by (\ref{smoothness of dcm lower}). Theorem \ref{learn dcm} further extends the result of \citet{suzuki2021deep} to fully-connected networks. Note that, when all $q_m=\infty$, $s^*$ can be arbitrarily close to $s^{**}$. Hence, we essentially recover the results of \citet{schmidthieber2020nonparametric,kohler2021rate}. But, in the general case, we do not get the minimax rate and we leave this as an interesting problem for future study.

\section{Proofs of approximation results}\label{sec: proofs of app}

Throughout this section, we let $b\ge 2$ be a fixed integer unless otherwise specified and suppress the dependence on $b$ to avoid superfluous notations. We denote $\Omega= [0,1)^d$ and notice that it is enough to establish our approximation results on $\Omega$. For any $\Bll =(\ell_1,\dots, \ell_d) \in \bN_0^d$, we can partition $\Omega$ into disjoint rectangles as follows:
\begin{equation}\label{partition}
\Omega = \bigcup_{\Bi\in I_\Bll} \Omega_{\Bll,\Bi}, \quad \Omega_{\Bll,\Bi} := \prod_{j=1}^d \left[ b^{-\ell_j} i_j, b^{-\ell_j} (i_j+1)\right),
\end{equation}
where the $d$-dimensional multi-index $\Bi$ is in the index set 
\[
I_\Bll:= \left\{ \Bi=(i_1,\dots,i_d): i_j\in \left[0:b^{\ell_j}-1\right] ,j\in [1:d] \right\}.
\]
Observing that the volume of each rectangle is $|\Omega_{\Bll,\Bi}| = b^{-|\Bll|}$, we have $|I_\Bll|=b^{|\Bll|}$ rectangles in this partition. For convenience, we will often transform the multi-index $\Bi \in I_\Bll$ to an index in $[0:b^{|\Bll|}-1]$ by the following function
\begin{equation}\label{index}
\ind_\Bll(\Bi) := \sum_{j=1}^d \left(\prod_{k=1}^{j-1} b^{\ell_k}\right) i_j,\quad \Bi \in I_\Bll.
\end{equation}

For $\Bk=(k_1,\dots,k_d)\in \bN_0^d$, we denote the linear space of polynomials with coordinate degree at most $\Bk$ by
\[
\cP_{\Bk} := \Span \left\{\Bx^\Br: \Br \in \bN_0^d, \Br \le \Bk \right\}.
\]
The space of piecewise polynomials with degree at most $\Bk$ subordinate to the partition (\ref{partition}) is denoted by 
\[
\cP_{\Bll, \Bk} := \left\{ f:\Omega \to \bR, f|_{\Omega_{\Bll,\Bi}} \in \cP_\Bk \mbox{ for all } \Bi\in I_\Bll \right\}.
\]
When $\Bk=(k,\dots,k)$, we denote $\cP_k := \cP_{\Bk}$ and $\cP_{\Bll, k} :=\cP_{\Bll, \Bk}$ for simplicity. Note that the space $\cP_{\Bll, \Bk}$ has a natural basis
\begin{equation}\label{basis}
\rho_{\Bll,\Bi}^\Bgamma(\Bx) := 
\begin{cases}
\prod_{j=1}^d (b^{\ell_j}x_j-i_j)^{\gamma_j},&  \Bx\in \Omega_{\Bll,\Bi},\\
0,& \Bx\notin \Omega_{\Bll,\Bi},
\end{cases}
\end{equation}
where $\Bgamma=(\gamma_1,\dots,\gamma_d)\in \bN_0^d$ is a multi-index with $\Bgamma \le \Bk$. Thus, the space $\cP_{\Bll, \Bk}$ is of dimension $ b^{|\Bll|} \prod_{j=1}^d (k_j+1)$.

We will approximate Besov functions by piecewise polynomials and construct neural networks to approximate these piecewise polynomials. However, since deep ReLU neural networks can only represent continuous piecewise linear functions, it is quite difficult to directly construct neural networks to approximate piecewise polynomials on the boundary of the partition (\ref{partition}) due to the possible discontinuity of these functions. To overcome this difficulty, we remove an arbitrarily small region (called trifling region) from $\Omega$, as suggested by \citep{lu2021deep,shen2020deep,shen2022optimal,siegel2023optimal,yang2025optimal}. Specially, given $\epsilon>0$, we define
\begin{equation}\label{good region}
\Omega_{\Bll,\epsilon}:= \bigcup_{\Bi\in I_\Bll} \Omega_{\Bll,\Bi,\epsilon},\quad \Omega_{\Bll,\Bi,\epsilon} := \prod_{j=1}^d
\begin{cases}
[b^{-\ell_j} i_j, b^{-\ell_j} (i_j+1)-\epsilon), & i_j<b^{\ell_j} -1,\\
[1-b^{-\ell_j}, 1), & i_j=b^{\ell_j} -1.
\end{cases} 
\end{equation}
Note that $\Omega_{\Bll,\epsilon}$ is nonempty when $\epsilon<\min_{j\in [1:d]} b^{-\ell_j}$. We will first approximate the target function on $\Omega_{\Bll,\epsilon}$ and then remove the trifling region by using the method from \citep{siegel2023optimal,yang2025optimal}.

\subsection{Preliminary constructions of neural networks}

In this section, we collect some useful constructions of neural networks. Similar results are widely used in the literature \citep{yarotsky2017error,yarotsky2018optimal,devore2021neural,lu2021deep,shen2022optimal,siegel2023optimal,yang2025optimal} and we will apply them as building blocks to derive approximation bounds for smooth functions. The following proposition gives basic properties of the function classes parameterized by neural networks. The proof can be found in  \citep{jiao2023approximation,siegel2023optimal}.

\begin{proposition}\label{basic constr}
Let $f_i\in \NN_{d_i,k_i}(W_i,L_i)$ for $i\in [1:n]$.
\begin{enumerate}[label=\textnormal{(\arabic*)},parsep=0pt]
\item If $d_1=d_2$, $k_1=k_2$ and $W_1\le W_2$, $L_1\le L_2$, then $\cN\cN_{d_1,k_1}(W_1,L_1) \subseteq \cN\cN_{d_2,k_2}(W_2,L_2)$.

\item \textnormal{\textbf{(Composition)}} If $k_1 = d_2$, then $f_2 \circ f_1 \in \cN\cN_{d_1,k_2}(\max\{W_1,W_2\},L_1+L_2)$. The result also holds when $f_1$ or $f_2$ is affine, if we view affine maps as neural networks with width $W=0$ and depth $L=0$.

\item \textnormal{\textbf{(Concatenation)}} If $d_1=d_2$, define $f(x):=(f_1(x),f_2(x))^\top$, then $f\in \cN\cN_{d_1,k_1+k_2}(W_1+W_2,\max\{L_1,L_2\})$.

\item \textnormal{\textbf{(Summation)}} If $d_i=d$ and $k_i=k$ for all $i=[1:n]$, then 
\[
\sum_{i=1}^n f_i \in \cN\cN_{d,k} \left(\sum_{i=1}^n W_i,
\max_{1\le i\le n} L_i\right) \cap \NN_{d,k}\left(\max_{1\le i\le n} W_i+2d+2k,\sum_{i=1}^n L_i\right).
\]
\end{enumerate}
\end{proposition}

It is useful to notice that there are two ways to construct a neural network that implements the sum of a collection of small networks, as shown by Part (4) of Proposition \ref{basic constr}. The first way is to concatenate the small networks in \emph{parallel} and compute the sum in the output layer. The second way is to compute the sum in \emph{sequential}, where we use $2d$ neurons to memorize the input and $2k$ neurons to memorize the partial sum (since $x=\sigma(x)-\sigma(-x)$ for any $x\in \bR$). In most cases, it is enough to use the first construction in the rest of the paper. When it is necessary to use the second construction, we will mention it explicitly. The next corollary gives a simple example of how to use these two constructions, which already indicates the trade-off between width and depth.

\begin{corollary}\label{basi constr co}
Suppose $\alpha,\beta \in \bN$ and $f_i\in \NN_{d,k}(W_i,L_i)$ with $W_i\le W$ and $L_i\le L$ for $i\in [1:\alpha \beta]$. Then, $f=\sum_{i=1}^{\alpha \beta} f_i \in \NN_{d,k}(\alpha W+2d+2k,\beta L)$.
\end{corollary}
\begin{proof}
For each $j\in [0:\beta -1]$, we have
\[
g_j := \sum_{i=1}^{\alpha} f_{j\alpha +i} \in \NN_{d,k}(\alpha W, L),
\]
by using Proposition \ref{basic constr} Part (4) in a parallel way since $\sum_{i=1}^\alpha W_{j\alpha +i} \le \alpha W$. As a consequence, $f=\sum_{j=0}^{\beta-1} g_j \in \NN_{d,k}(\alpha W+2d+2k,\beta L)$ by using Proposition \ref{basic constr} Part (4) in a sequential way.
\end{proof}

Next proposition shows that we can approximately compute the index (\ref{index}) of an input $\Bx \in \Omega_{\Bll,\Bi}$ by neural networks. Hence, we can roughly discretize the input domain.

\begin{proposition}\label{app partition}
Let $\Bll \in \bN_0^d$ and $0<\epsilon<\min_{j\in [1:d]} b^{-\ell_j}$. Then, for any $L\in \bN$, there exists $\phi_\Bll \in \NN_{d,1}(2\sum_{j=1}^d b^{\lceil \ell_j/L \rceil},L)$ such that
\[
\phi_\Bll(\Bx) = \ind_\Bll(\Bi),\quad \forall \Bx\in \Omega_{\Bll,\Bi,\epsilon},
\]
where $\ind_{\Bll}(\Bi)$ is defined by (\ref{index}) and $\Omega_{\Bll,\Bi,\epsilon}$ is defined by (\ref{good region}).
\end{proposition}
\begin{proof}
The one dimensional case $d=1$ has been proven by \citet[Lemma 4.4]{yang2025optimal}. Specifically, for $\ell\in \bN_0$, let 
\[
\Omega_{\ell,\epsilon}= \bigcup_{i=0}^{b^\ell-1} \Omega_{\ell,i,\epsilon},\quad \Omega_{\ell,i,\epsilon} =
\begin{cases}
[b^{-\ell} i, b^{-\ell} (i+1)-\epsilon), & i<b^{\ell} -1,\\
[1-b^{-\ell}, 1), & i=b^{\ell} -1.
\end{cases} 
\]
\citet[Lemma 4.4]{yang2025optimal} showed the existence of $\phi_\ell\in \NN_{1,1}(2b^{\lceil \ell/L \rceil},L)$ such that $\phi_\ell(x) = i$ for $x\in \Omega_{\ell,i,\epsilon}$ when $0<\epsilon<b^{-\ell}$. Observing that $\Omega_{\Bll,\Bi,\epsilon}=\prod_{j=1}^d \Omega_{\ell_j,i_j,\epsilon}$, we define 
\[
\phi_\Bll(\Bx) = \sum_{j=1}^d \left(\prod_{k=1}^{j-1} b^{\ell_k}\right) \phi_{\ell_j}(x_j).
\]
Then, $\phi_\Bll(\Bx) = \ind_\Bll(\Bi)$ for $\Bx\in \Omega_{\Bll,\Bi,\epsilon}$ by the definition of $\ind_\Bll(\Bi)$ under the condition that $0<\epsilon<\min_{j\in [1:d]} b^{-\ell_j}$. By Proposition \ref{basic constr}, it is easy to see that $\phi_\Bll \in \NN_{d,1}(2\sum_{j=1}^d b^{\lceil \ell_j/L \rceil},L)$, since $\phi_\Bll$ is a linear combination of $\phi_{\ell_j}$ for $j\in [1:d]$.
\end{proof}

In order to remove the trifling region and derive approximation bounds on the whole input domain, we will need the following technical construction from \citet[Corollary 13]{siegel2023optimal}. It shows that we can use neural networks to select any order statistic.

\begin{proposition}\label{select median}
Let $d=2^k$ for some $k\in \bN$. For each integer $1\le j\le d$, there exists $\psi_j\in \NN_{d,1}(4d,k(k+1)/2)$ such that $\psi_j(\Bx) = x_{(j)}$, where $x_{(j)}$ is the $j$-th largest entry of $\Bx\in \bR^d$.
\end{proposition}

\subsection{Approximation of piecewise polynomials}

In this section, we study how well piecewise polynomials can be approximated by ReLU neural networks. Following the idea in the seminal work of \citet{yarotsky2017error}, our starting point is the approximation of the product function $(x,y) \mapsto xy$. Once we can approximate the product, we are able to approximate monomials and hence any polynomials. We remark that \citet[Proposition 3]{yarotsky2017error} constructed a ReLU network with depth $\cO(\log (1/\epsilon))$ which approximates the product function with error at most $\epsilon$. \citet[Lemma A.3]{petersen2018optimal} showed that this can be done using network with fixed depth by applying the trade-off between depth and the number of parameters. \citet[Lemma 5.2]{lu2021deep} further characterized the trade-off between width and depth by showing the existence of functions in $\NN(9W,L)$ that approximate the product with error $6W^{-L}$. We will use the following construction from \citet[Lemma 4.11]{yang2025optimal}, which is a slight modification of \citet[Lemma 5.2]{lu2021deep}.

\begin{lemma}\label{app of product}
For any integers $k\ge 4$ and $L\ge 1$, there exists $f_{k,L}\in \NN(3k2^k+3,L)$ such that $f_{k,L}: [-1,1]^2 \to [-1,1]$ and
\[
|f_{k,L}(x,y) - xy| \le 2^{-2kL-1}, \quad \forall x,y\in [-1,1].
\]
\end{lemma}

Applying Lemma \ref{app of product} iteratively, one can show that polynomials can be efficiently approximated by neural networks and the approximation error decays exponentially when the depth increases. Next, we consider the approximation of piecewise polynomials 
\[
\sum_{\Bgamma\le \Bk, \Bi \in I_\Bll} a_{\Bi,\Bgamma}  \rho_{\Bll,\Bi}^\Bgamma(\Bx) \in \cP_{\Bll, \Bk}
\]
corresponding to the partition (\ref{partition}) by neural network $\NN(W,L)$. When the partition is isotropic, i.e., all components of $\Bll$ are equal, this approximation problem has been widely studied \citep{yarotsky2017error,lu2021learning,shen2022optimal,siegel2023optimal,yang2025optimal}. Since there are $N=|I_\Bll|$ pieces, a straightforward approach is to approximate each piece by a small network and concatenate these networks to get the desired approximator with width $W\lesssim N$, which has been done in \citet{yarotsky2017error}. In order to get supper approximation rate, we can discretize the problem by using the map $\Bx \in \Omega_{\Bll,\Bi} \mapsto \Bi$ and reduce it to the interpolation problem $\Bi \in I_\Bll \mapsto \Ba_\Bi=(a_{\Bi,\Bgamma})_{\Bgamma\le \Bk}$ that computes the coefficients of the polynomials. As shown by \citet{lu2021deep,shen2022optimal}, if $\Ba_\Bi$ are properly discretized, this interpolation problem can be roughly achieved by a network with size $W^2L^2 \lesssim N$ so that we improve the approximation rate (note that polynomials can be efficiently approximated, so the size of the network is mainly determined by the complexity of the interpolation). 

However, when we study the approximation of Besov functions $\cB^s_q$ with error measured in the $L^p$ norm, the above analysis can only be applied to the linear approximation regime $p\le q$ \citep{devore1998nonlinear,siegel2023optimal}. In the nonlinear regime $p>q$, where adaptive nonuniform grid is needed, we will decompose the target function into piecewise polynomials with multiple resolution (see Lemmas \ref{anisotropic decomp} and \ref{mixed decomp} below). In order to derive optimal approximation rate, we need to take into account the sparsity of the coefficients of these polynomials \citep{siegel2023optimal,yang2025optimal}. This can be done by imposing a norm constraint $\|\Ba\|_1\le M$ for the coefficients $\Ba=(\Ba_\Bi)_{\Bi\in I_\Bll}$ in our interpolation problem. The next lemma from \citet[Theorem 4.6]{yang2025optimal} implies that, if the coefficients are properly discretized, these data can be interpolated by a network with size
\[
W^2L^2 \lesssim
\begin{cases}
M(1+\log_2(N/M) ), &\mbox{if }N\ge M, \\
N(1+\log_2(M/N) ), &\mbox{if }N\le M.
\end{cases}
\]
Hence, we can get better approximation in the sparse regime $N\ge M$.

\begin{lemma}\label{rep of vector}
Let $N,M \in \bN$ and $\Bu=(u_1,\dots,u_N)^\top\in \bZ^N$ satisfy $\|\Bu\|_1\le M$.
\begin{enumerate}[label=\textnormal{(\arabic*)},parsep=0pt]
\item If $N\ge M$, then for any $S,T \in \bN$, there exists $g\in \NN(W,L)$ with 
\[
W= 22 \max \left\{ \left\lceil \frac{\sqrt{M}}{S \sqrt{T+2}} \right\rceil, \left\lceil \left(\frac{N}{M}\right)^{1/T} \right\rceil \right\} +10, \quad L= 4S(T+2),
\]
such that $g(n) = u_n$ for $n=1,\dots,N$.

\item If $N\le M$, then for any $S,T \in \bN$, there exists $g\in \NN(W,L)$ with 
\[
W= 22 \max \left\{ \left\lceil \frac{M}{S \sqrt{N(T+2)}} \right\rceil, \left\lceil \left(\frac{M}{N}\right)^{1/T} \right\rceil \right\} +12, \quad L= 4 \left\lceil \frac{SN}{M} \right\rceil (T+2),
\]
such that $g(n) = u_n$ for $n=1,\dots,N$.
\end{enumerate}
\end{lemma}

Next, we combine Lemmas \ref{app of product} and \ref{rep of vector} with the idea described above to estimate the approximation error of piecewise polynomials. The following lemma generalizes \citet[Proposition 4.12]{yang2025optimal} from isotropic grid to the anisotropic case, which will be a key component in our approximation theory of anisotropic and mixed Besov functions in Sections \ref{sec: app of anisotropic} and \ref{sec: app of mixed}.

\begin{lemma}\label{app of poly}
Let $\Bll,\Bk \in \bN_0^d$ and $0<\epsilon<\min_{j\in [1:d]} b^{-\ell_j}$. Suppose that $f\in \cP_{\Bll, \Bk}$ is expanded in terms of the bases $\rho_{\Bll,\Bi}^\Bgamma$ defined by (\ref{basis}) with coefficients $\Ba=(a_{\Bi,\Bgamma})_{\Bgamma\le \Bk, \Bi \in I_\Bll}$, i.e.,
\[
f(\Bx) = \sum_{\Bgamma\le \Bk, \Bi \in I_\Bll} a_{\Bi,\Bgamma} \rho_{\Bll,\Bi}^\Bgamma(\Bx).
\]
Let $1\le q\le p\le \infty$ and choose a parameter $\delta > 0$. Denote the $q$-norm of $\Ba$ by (with the standard modification when $q=\infty$)
\[
\|\Ba\|_q = \left( \sum_{\Bgamma\le \Bk, \Bi \in I_\Bll} |a_{\Bi,\Bgamma}|^q \right)^{1/q}.
\]
The following approximation results hold for some constants $C:=C(p,q,d,\Bk,b)$ depending only on $p,q,d,\Bk$ and the base $b$. 
\begin{enumerate}[label=\textnormal{(\arabic*)},parsep=0pt]
\item If $\delta q \le |\Bll|$, then for any $S,T,W_0,L_0 \in \bN$, there exists $g\in \NN(W,L)$ with 
\begin{align*}
W &\le C \max \left\{ \frac{b^{\delta q/2}}{S \sqrt{T}}, b^{(|\Bll| -\delta q)/T}, \sum_{j=1}^d b^{\ell_j/L_0 }, W_0 2^{W_0} \right\}, \\
L &\le C(ST+L_0),
\end{align*}
such that
\[
\|f-g\|_{L^p(\Omega_{\Bll,\epsilon})} \le C \left(b^{\delta q/p -|\Bll|/p-\delta} + 4^{-W_0L_0}\right) \|\Ba\|_q.
\]
\item If $\delta q \ge |\Bll|$, then for any $S,T,W_0,L_0 \in \bN$, there exists $g\in \NN(W,L)$ with 
\begin{align*}
W &\le C \max \left\{ \frac{b^{\delta+|\Bll|/2-|\Bll|/q}}{S \sqrt{T}} , b^{(\delta-|\Bll|/q)/T}, \sum_{j=1}^d b^{\ell_j/L_0 }, W_0 2^{W_0} \right\}, \\
L &\le C \left(\left\lceil b^{-\delta + |\Bll|/q} S \right\rceil T+L_0 \right),
\end{align*}
such that 
\[
\|f-g\|_{L^p(\Omega_{\Bll,\epsilon})} \le C \left(b^{-\delta} + 4^{-W_0L_0}\right)\|\Ba\|_q.
\]
\end{enumerate}
\end{lemma}

\begin{proof}
We modify the proof of \citet[Proposition 4.12]{yang2025optimal}. Without loss of generality, we can assume that $\|\Ba\|_q\le 1$. Notice that $f = \sum_{\Bgamma\le \Bk} f_\Bgamma$ with $f_\Bgamma(\Bx) = \sum_{\Bi \in I_\Bll} a_{\Bi,\Bgamma} \rho_{\Bll,\Bi}^\Bgamma(\Bx)$. By the triangle inequality, it is enough to prove the approximation result for each $f_\Bgamma$ and construct the desired network using Proposition \ref{basic constr} Part (4) at the expense of larger constants. Hence, we assume that $f=f_\Bgamma$ and $\Ba=(a_{\Bi})_{\Bi \in I_\Bll}$ with $a_{\Bi} = a_{\Bi,\Bgamma}$ in the following analysis.

In order to apply Lemma \ref{rep of vector}, we need to discretize the coefficients $\Ba$. To do this, for a given $\delta>0$, we define $\widetilde{\Ba} = (\widetilde{a}_{\Bi})_{\Bi\in I_\Bll}$, where 
\[
\widetilde{a}_{\Bi} := b^{-\delta} \sgn(a_\Bi) \lfloor b^\delta |a_\Bi| \rfloor.
\]
We can estimate the discretization error as follows. Since $a_{\Bi} := b^{-\delta} \sgn(a_\Bi) b^\delta |a_\Bi| $, it is easy to get $\|\Ba-\widetilde{\Ba}\|_\infty \le b^{-\delta}$ and $\|\Ba - \widetilde{\Ba}\|_q\le \|\Ba\|_q\le 1$. Recall that $|I_\Bll|= b^{|\Bll|}$, which implies 
\[
\|\Ba - \widetilde{\Ba}\|_p \le |I_\Bll|^{1/p} \|\Ba-\widetilde{\Ba}\|_\infty \le b^{|\Bll|/p-\delta}.
\]
Since we assume that $p\ge q$, we also have
\[
\|\Ba-\widetilde{\Ba}\|_p \le \|\Ba-\widetilde{\Ba}\|_q^{q/p} \|\Ba-\widetilde{\Ba}\|_\infty^{1-q/p} \le b^{\delta q/p -\delta}.
\]
Combining the above two inequalities, we conclude that
\begin{equation}\label{bound for coe}
\|\Ba-\widetilde{\Ba}\|_p \le b^{-\delta} \min\{ b^{|\Bll|/p}, b^{\delta q/p} \}.
\end{equation}
Let $\ind_\Bll(\Bi)$ be the index corresponding to $\Bi \in I_\Bll$ as defined by (\ref{index}). We define the vector $\Bu\in \bZ^{b^{|\Bll|}}$ whose $\ind_\Bll(\Bi)+1$-entry is given by $u_{\ind_\Bll(\Bi)+1} = b^\delta \widetilde{a}_{\Bi} = \sgn(a_\Bi) \lfloor b^\delta |a_\Bi| \rfloor$. We will apply Lemma \ref{rep of vector} to $\Bu$ below. In order to do this, we estimate $\|\Bu\|_1$ as follows. Observing that $\|\Bu\|_q \le b^\delta \|\Ba\|_q \le b^\delta$, one can obtain $|\{i:u_i\neq 0\}| \le \min\{b^{\delta q}, b^{|\Bll|} \}$. Using H\"older's inequality, we have
\[
\|\Bu\|_1 \le |\{i:u_i\neq 0\}|^{1-1/q} \|\Bu\|_q \le \min\{b^{\delta q}, b^{\delta+ |\Bll|-|\Bll|/q} \}.
\]

We construct the desired neural network $g$ as follows. We first apply networks $\phi_{\ell_j}\in \NN_{1,1}(2b^{\lceil \ell_j/{L_0} \rceil},L_0)$ and $\phi_\Bll \in \NN_{d,1}(2\sum_{j=1}^d b^{\lceil \ell_j/{L_0} \rceil},L_0)$ in Proposition \ref{index} to identify the index of the input:
\begin{equation}\label{compute ind}
\Bx \to 
\begin{pmatrix}
\phi_{\ell_1}(x_1) \\
\vdots \\
\phi_{\ell_d}(x_d) \\
\phi_{\Bll}(\Bx) \\
\Bx
\end{pmatrix}
\to 
\begin{pmatrix}
b^{\ell_1} x_1 - \phi_{\ell_1}(x_1) \\
\vdots \\
b^{\ell_d} x_d - \phi_{\ell_d}(x_d) \\
\phi_{\Bll}(\Bx) +1
\end{pmatrix}
\xlongequal{\text{if } \Bx \in\Omega_{\Bll,\Bi,\epsilon}} 
\begin{pmatrix}
b^{\ell_1} x_1 - i_1 \\
\vdots \\
b^{\ell_d} x_d - i_d \\
\ind_\Bll(\Bi) +1
\end{pmatrix}.
\end{equation}
By Proposition \ref{basic constr}, it is easy to see that the above mapping can be implemented by a network with width $W_1 \le C \sum_{j=1}^d b^{\ell_j/{L_0}}$ and depth $L_1=L_0$. Next, we apply Lemma \ref{rep of vector} to the vector $\Bu$ with $N=b^{|\Bll|}$ and $M=\min\{b^{\delta q}, b^{\delta+ |\Bll|-|\Bll|/q} \}$, which gives us a network $\varphi \in \NN(W_2,L_2)$ such that $\varphi(\ind_\Bll(\Bi) +1) = u_{\ind_\Bll(\Bi) +1} = b^\delta \widetilde{a}_{\Bi}$. Furthermore, if $\delta q \le |\Bll|$, then $M=b^{\delta q}\le b^{|\Bll|} = N$ and we can choose
\[
W_2\le C \max \left\{ \frac{b^{\delta q/2}}{S \sqrt{T}} , b^{(|\Bll| -\delta q)/T} \right\}, \quad L_2\le CST.
\]
If $\delta q \ge |\Bll|$, then $M = b^{\delta +|\Bll| - |\Bll|/q} \ge b^{|\Bll|} =N$ and we can choose
\[
W_2 \le C \max \left\{ \frac{b^{\delta+|\Bll|/2-|\Bll|/q}}{S \sqrt{T}}, b^{(\delta-|\Bll|/q)/T} \right\}, \quad L_2 \le C \left\lceil b^{-\delta + |\Bll|/q} S \right\rceil T.
\]
We compose $b^{-\delta}\varphi(\cdot)$ with the last component of (\ref{compute ind}) and use the identity map $\Id(t)=\sigma(t)-\sigma(-t)$ to ``memorize'' the first $d$ components of (\ref{compute ind}). Then, by Proposition \ref{basic constr}, we get a network $h\in \NN(\max\{W_1,W_2+2d\},L_1+L_2)$, which satisfies
\[
h(\Bx) =
\begin{pmatrix}
b^{\ell_1} x_1 - i_1 \\
\vdots \\
b^{\ell_d} x_d - i_d \\
\widetilde{a}_\Bi
\end{pmatrix}
\in [-1,1]^{d+1}, \quad \forall \Bx \in \Omega_{\Bll,\Bi,\epsilon}.
\]
Finally, we construct a neural network $P_\Bgamma$ to approximate the polynomial $(\Bz,a) \mapsto a \Bz^\Bgamma$ for $\Bz\in [-1,1]^d$ and $a\in [-1,1]$. To do this, we let $f_{W_0+3,L_0} \in \NN_{2,1}(24(W_0+3)2^{W_0}+3,L_0)$ be the network from Lemma \ref{app of product} and define
\[
P_j:
\begin{pmatrix}
\Bz \\
a
\end{pmatrix}
\to
\begin{pmatrix}
\Bz \\
f_{W_0+3,L_0}(z_j,a)
\end{pmatrix},
\quad j\in [1:d].
\]
We construct $P_\Bgamma$ by composing $\gamma_j$ copies of $P_j$ and then applying an affine map which selects the last coordinate. By Proposition (\ref{basic constr}), we know that $P_\Bgamma \in \NN_{d+1,1}(W_3,L_3)$ with $W_3\le CW_0 2^{W_0}$ and $L_3\le |\Bgamma| L_0$. 
Composing $P_\Bgamma$ with $h$, we obtain the desired network $g\in \NN(W,L)$ such that $g(\Bx) = P_\gamma(\Bz,\widetilde{a}_\Bi)$ with $\Bz=(b^{\ell_1} x_1 - i_1,\dots,b^{\ell_d} x_d - i_d)^\top$ for $\Bx \in \Omega_{\Bll,\Bi,\epsilon}$. By Proposition (\ref{basic constr}), its width $W\le C \max\{W_1,W_2,W_3\}$ and depth $L= L_1+L_2+L_3$.

It remains to estimate the approximation error of $g$. Applying the approximation bound from Lemma \ref{app of product} iteratively, we get
\[
\left|P_\gamma(\Bz,a) - a\Bz^{\Bgamma} \right| \le \sum_{j=1}^d \gamma_j 2^{-2(W_0+3)L_0-1} \le C 4^{-W_0L_0}.
\]
Since the bases $\rho_{\Bll,\Bi}^\Bgamma$ have disjoint supports, we have (with obvious modification for $p=\infty$)
\begin{align*}
\|f-g\|_{L^p(\Omega_{\Bll,\epsilon})}^p &= \sum_{\Bi \in I_\Bll} \int_{\Omega_{\Bll,\Bi,\epsilon}} \left| a_\Bi \prod_{j=1}^d (b^{\ell_j}x_j-i_j)^{\gamma_j} - P_\Bgamma(\Bz,\widetilde{a}_\Bi) \right|^p d\Bx \\
&= \sum_{\Bi \in I_\Bll} \int_{\Omega_{\Bll,\Bi,\epsilon}} \left| a_\Bi \Bz^\Bgamma - P_\Bgamma(\Bz,\widetilde{a}_\Bi) \right|^p d\Bx \\
&\le 2^{p-1} \sum_{\Bi \in I_\Bll} \int_{\Omega_{\Bll,\Bi,\epsilon}} |a_\Bi - \widetilde{a}_\Bi|^p|\Bz^\Bgamma|^p + \left| \widetilde{a}_\Bi \Bz^\Bgamma - P_\Bgamma(\Bz,\widetilde{a}_\Bi) \right|^p d\Bx \\
&\le Cb^{-|\Bll|} \|\Ba-\widetilde{\Ba}\|_p^p + C4^{-pW_0L_0},
\end{align*}
where we use $|\Omega_{\Bll,\Bi}| = b^{-|\Bll|}$ in the last inequality. By the bound (\ref{bound for coe}), 
\[
\|f-g\|_{L^p(\Omega_{\Bll,\epsilon})} \le C \left( \min\{ b^{-\delta}, b^{\delta q/p-|\Bll|/p-\delta} \} + 4^{-W_0L_0}\right),
\]
which completes the proof.
\end{proof}

\subsection{Approximation of anisotropic Besov functions}\label{sec: app of anisotropic}

This section analyzes the approximation error of anisotropic Besov functions using neural networks and gives a proof of Theorem \ref{app of anisotropic}. Given an anisotropic smoothness $\Bs = (s_1,\dots, s_d) \in (0,\infty)^d$, we use the partition (\ref{partition}) with the following choice of $\Bll =(\ell_1,\dots, \ell_d) \in \bN_0^d$ in this section:
\begin{equation}\label{grid for anisotropic}
\ell_j = \lfloor \ell \tilde{s}/s_j \rfloor, \quad \ell \in \bN_0,
\end{equation}
where $\tilde{s}$ is the mean smoothness defined by (\ref{mean smoothness}). Since this partition only depends on the resolution parameter $\ell$ (and $\Bs$), we will denote $\Omega_{\ell,\Bi}:=\Omega_{\Bll,\Bi}$, $I_\ell:=I_\Bll$ and similarly for other notations. We note that any rectangle $\Omega_{\ell,\Bi}$ is a subset of some rectangle $\Omega_{\ell-1,\Bi^-}$ for $\ell\ge 1$. By the definition of $\tilde{s}$, we have $\sum_{j=1}^d \tilde{s}/s_j =1$, which implies that the volume of $\Omega_{\ell,\Bi}$ satisfies
\[
|\Omega_{\ell,\Bi}| = b^{-|\Bll|} = b^{-\sum_{j=1}^d \lfloor \ell \tilde{s}/s_j \rfloor}
\begin{cases}
\ge b^{-\sum_{j=1}^d \ell \tilde{s}/s_j } = b^{-\ell}, \\
\le b^{-\sum_{j=1}^d (\ell \tilde{s}/s_j-1) }= b^d b^{-\ell}.
\end{cases}
\]
Hence, $|\Omega_{\ell,\Bi}| \asymp b^{-\ell}$ and we have roughly $b^\ell$ rectangles in the resolution $\ell$.

\subsubsection{Approximation by piecewise polynomials}

The neural network approximation of piecewise polynomials has been analyzed in Lemma \ref{app of poly}. We now consider how well anisotropic Besov functions can be approximated by piecewise polynomials. Let us begin with the following well-known generalization of Whitney's Theorem in approximation theory \citep{brudnyi1970multidimensional,dahmen1980multidimensional}.

\begin{lemma}\label{whitney}
Let $\Bk=(k_1,\dots,k_d)\in \bN^d$ and $Q\subseteq \bR^d$ be a rectangle with side length vector $\Bdelta=(\delta_1,\dots,\delta_d)$. Then for any $f\in L^q(Q)$, there exists a polynomial $g\in \cP_{\Bk -\boldsymbol{1}}$ such that 
\[
\| f- g\|_{L^q(Q)} \le C \sum_{j=1}^d \omega_{k_j}^j(f,\delta_j,Q)_q,
\]
for some constant $C>0$ independent of $f$ and $Q$.
\end{lemma}

Lemma \ref{whitney} relates the approximation error to the modulus of smoothness $\omega_k^j$. However, $\omega_k^j$ is not satisfactory when we want to add local estimates over different rectangles. A useful tool to resolve this problem is the averaged modulus of smoothness. For $t> 0$, $1\le q< \infty$ and $j\in [1:d]$, we define
\[
\widetilde{\omega}_k^j(f,t,\Omega)_q := \frac{1}{t} \int_0^t \|\Delta_h^{k,j} f\|_{L^q(\Omega(k,h,j))} dh.
\]
Then, we have the following summing property: if $\Omega = \cup_{i=1}^N \Omega_i$ is partitioned into disjoint rectangles $\Omega_i$, then, for $f\in L^q(\Omega)$, by H\"older inequality,
\begin{align*}
\sum_{i=1}^N \widetilde{\omega}_k^j(f,t,\Omega_i)_q^q &\le \sum_{i=1}^N \frac{1}{t} \int_0^t \|\Delta_h^{k,j} f\|_{L^q(\Omega_i(k,h,j))}^q dh \\
&\le \frac{1}{t} \int_0^t \|\Delta_h^{k,j} f\|_{L^q(\cup_{i=1}^N \Omega_i(k,h,j))}^q dh \\
&\le \widetilde{\omega}_k^j(f,t,\Omega)_q^q.
\end{align*}
Another useful property of the averaged modulus of smoothness $\widetilde{\omega}_k^j$ is its equivalence with the ordinary modulus of smoothness $\omega_k^j$.

\begin{lemma}\label{modulus equ}
Let $Q\subseteq \bR^d$ be a rectangle with side length vector $\Bdelta=(\delta_1,\dots,\delta_d)$. Then, for any $f\in L^q(Q)$ with $1\le q<\infty$, $k\in \bN$, $j\in [1:d]$ and $0<t\le\delta_j/(4k)$, we have
\[
\widetilde{\omega}_k^j(f,t,Q)_q \le \omega_k^j(f,t,Q)_q \le C \widetilde{\omega}_k^j(f,t,Q)_q,
\]
where $C>0$ is a constant depending only on $k$.
\end{lemma}
\begin{proof}
The equivalence is well known for one dimensional case \citep[Chapter 6, Lemma 5.1]{devore1993constructive}. We can use similar argument to prove the high dimensional case $d\ge 2$. Since $\|\Delta_h^{k,j} f\|_{L^q(Q(k,h,j))} \le \omega_k^j(f,t,Q)_q$ for $0\le h\le t$, we immediately get $\widetilde{\omega}_k^j(f,t,Q)_q \le \omega_k^j(f,t,Q)_q$. For the other inequality, we note that one can obtain an identity analogous to \citep[Page 184, Eq. (5.3)]{devore1993constructive} for differences:
\begin{equation}\label{difference identity}
\Delta_h^{k,j} f(\Bx) = \sum_{m=1}^k (-1)^m \binom{k}{m} \left[ \Delta_{ms}^{k,j} f(\Bx+mh\Be_j) - \Delta_{h+ms}^{k,j} f(\Bx) \right],
\end{equation}
for every $s\in \bR$ and $f$ defined on $\bR^d$. Suppose $Q=\prod_{i=1}^d[0,\delta_i]$ for simplicity and denote $Q^{-}_j=\prod_{i=1}^d[0,c_i\delta_i]$ with $c_i=1/2$ for $i=j$ and $c_i=1$ for $i\neq j$. Then, for any $0< h\le t/2\le \delta_j/(8k^2)$ and $0\le s\le t$, $\Bx\in Q^{-}_j$ implies that $\Bx+mh\Be_j \in Q(k,ms,j)$ and $\Bx \in Q(k,h+ms,j)$. Thus, integrating (\ref{difference identity}) gives
\[
\| \Delta_h^{k,j} f\|_{L^q(Q^{-}_j)} \le \sum_{m=1}^k (-1)^m \binom{k}{m} \left[ \|\Delta_{ms}^{k,j} f\|_{L^q(Q(k,ms,j))} - \|\Delta_{h+ms}^{k,j} f\|_{L^q(Q(k,h+ms,j))} \right].
\]
Taking average with respect to $s\in [0,\beta t]$ with $\beta=1-1/(2k)$ and making change of variables, we get
\begin{align*}
\| \Delta_h^{k,j} f\|_{L^q(Q^{-}_j)} &\lesssim \frac{1}{t} \sum_{m=1}^k \left( \int_0^{m\beta t} \|\Delta_{u}^{k,j} f\|_{L^q(Q(k,u,j))} du + \int_h^{h+m\beta t} \|\Delta_{u}^{k,j} f\|_{L^q(Q(k,u,j))} du \right) \\
&\lesssim \frac{1}{t} \int_0^{kt} \|\Delta_{u}^{k,j} f\|_{L^q(Q(k,u,j))} du \lesssim \widetilde{\omega}_k^j(f,kt,Q)_q.
\end{align*}
By symmetry, one can obtain similar bound for $\| \Delta_{-h}^{k,j} f\|_{L^q(Q^{+}_j)}$ with $Q^{+}_j = Q\setminus Q^{-}_j$. Since this is true for all $0<h\le t/2$, the same bound holds for $\omega_k^j(f,t/2,Q)_q$. Thus,
\[
\omega_k^j(f,kt,Q)_q \le (2k+1)^k \omega_k^j(f,t/2,Q)_q \lesssim \widetilde{\omega}_k^j(f,t,Q)_q,
\]
for any $0<t\le \delta_j/(4k^2)$.
\end{proof}

The following lemma decomposes anisotropic Besov functions into piecewise polynomials with different resolutions. We also estimate the coefficients of these piecewise polynomials so that Lemma \ref{app of poly} can be applied. The approximation error can be obtained by truncating the decomposition.

\begin{lemma}\label{anisotropic decomp}
Let $1\le q\le p\le \infty$, $\Bs\in (0,\infty)^d$ and $\Bk\in \bN^d$ with $k_j=\lfloor s_j \rfloor +1$. If $\tilde{s}>1/q-1/p$, then we can decompose any $f\in \cB_q^\Bs(\Omega)$ as 
\[
f = \sum_{\ell=0}^\infty f_\ell,\quad f_\ell(\Bx) = \sum_{\Bi \in I_\ell, \Bgamma \le \Bk -\boldsymbol{1}} a_{\ell,\Bi,\Bgamma} \rho_{\ell,\Bi}^\Bgamma(\Bx) \in \cP_{\ell,\Bk-\boldsymbol{1}},
\]
where $\rho_{\ell,\Bi}^\Bgamma$ is the basis (\ref{basis}), $a_{\ell,\Bi,\Bgamma}\in \bR$ and the convergence is in $L^p(\Omega)$. Furthermore, if we let $\Ba_\ell = (a_{\ell,\Bi,\Bgamma})_{\Bi \in I_\ell, \Bgamma \le \Bk -\boldsymbol{1}}$ be the vector of coefficients of $f_\ell$, then we have the following estimates
\begin{align*}
\|\Ba_\ell\|_q &\lesssim b^{(1/q-\tilde{s})\ell} \|f\|_{\cB_q^\Bs(\Omega)},\\
\|f_\ell\|_{L^p(\Omega)} &\lesssim b^{(1/q-1/p-\tilde{s})\ell} \|f\|_{\cB_q^\Bs(\Omega)},
\end{align*}
where the implied constants depend only on $p, q, \Bs, d$ and the base $b$.
\end{lemma}
\begin{proof}
Let us first consider the case $1\le q<\infty$. Let $g_\ell$ be the $L^q$-projection of $f\in L^q(\Omega)$ onto the space of piecewise polynomials $\cP_{\ell,\Bk-\boldsymbol{1}}$ in the sense that
\[
g_\ell \in \argmin_{g\in \cP_{\ell,\Bk-\boldsymbol{1}}} \|f-g\|_{L^q(\Omega)}.
\]
Since $\cP_{\ell,\Bk-\boldsymbol{1}}$ is a finite dimensional subspace of $L^q(\Omega)$, we know that the above best approximation always exits \citep[Chapter 3, Theorem 1.1]{devore1993constructive}. Applying Lemma \ref{whitney} to each rectangle $\Omega_{\ell,\Bi}$, we get
\[
\| f- g_\ell\|_{L^q(\Omega_{\ell,\Bi})} \lesssim \sum_{j=1}^d \omega_{k_j}^j(f,b^{-\ell_j},\Omega_{\ell,\Bi})_q.
\]
If we let $\lambda = \min_{j\in [1:d]} 1/(5k_j)$, then the inequality $\omega_{k_j}^j(f,t,\Omega)_q \le (1/\lambda+1)^{k_j} \omega_{k_j}^j(f,\lambda t,\Omega)_q$ and Lemma \ref{modulus equ} imply that
\begin{align*}
\omega_{k_j}^j(f,b^{-\ell_j},\Omega_{\ell,\Bi})_q &\le (1/\lambda+1)^{k_j} \omega_{k_j}^j(f,\lambda b^{-\ell_j},\Omega_{\ell,\Bi})_q \\
&\lesssim \widetilde{\omega}_{k_j}^j(f,\lambda b^{-\ell_j},\Omega_{\ell,\Bi})_q.
\end{align*}
As a consequence, by the summing property of $\widetilde{\omega}_{k_j}^j$,
\begin{align*}
\|f-g_\ell\|_{L^q(\Omega)}^q &= \sum_{\Bi\in I_\ell} \| f- g_\ell\|_{L^q(\Omega_{\ell,\Bi})}^q \lesssim \sum_{\Bi\in I_\ell} \left( \sum_{j=1}^d \widetilde{\omega}_{k_j}^j(f,\lambda b^{-\ell_j},\Omega_{\ell,\Bi})_q \right)^q \\
&\lesssim \sum_{j=1}^d \sum_{\Bi\in I_\ell} \widetilde{\omega}_{k_j}^j(f,\lambda b^{-\ell_j},\Omega_{\ell,\Bi})_q^q \le \sum_{j=1}^d \widetilde{\omega}_{k_j}^j(f,\lambda b^{-\ell_j},\Omega)_q^q \\
&\lesssim \sum_{j=1}^d \omega_{k_j}^j(f,b^{-\ell_j},\Omega)_q^q,
\end{align*}
where we use Lemma \ref{modulus equ} and the non-decreasing of $\omega_{k_j}^j$ in the last inequality. For any $f\in \cB_q^\Bs(\Omega)$, by the definition of the semi-norm $|f|_{\cB^{\Bs,j}_{q,\infty}(\Omega)}$, we have
\[
\|f-g_\ell\|_{L^q(\Omega)}^q \lesssim \sum_{j=1}^d \omega_{k_j}^j(f,b^{-\ell_j},\Omega)_q^q \le \sum_{j=1}^d b^{-s_j \ell_j q} |f|_{\cB^{\Bs,j}_{q,\infty}(\Omega)}^q.
\]
Recall that $\ell_j =\lfloor \ell \tilde{s}/s_j \rfloor$, which implies $s_j \ell_j \ge \tilde{s} \ell - s_j$ and hence
\begin{equation}\label{temp1}
\|f-g_\ell\|_{L^q(\Omega)} \lesssim \left(\sum_{j=1}^d b^{-\tilde{s} \ell q} |f|_{\cB^{\Bs,j}_{q,\infty}(\Omega)}^q \right)^{1/q} \lesssim b^{-\tilde{s} \ell} \|f\|_{\cB_q^\Bs(\Omega)}.
\end{equation}

We let $f_0=g_0 \in \cP_{0,\Bk-\boldsymbol{1}}$ and $f_\ell = g_\ell - g_{\ell-1} \in \cP_{\ell,\Bk-\boldsymbol{1}}$ for $\ell\in \bN$. Then, we have the decomposition
\[
f = g_0 + \sum_{\ell=1}^\infty (g_\ell - g_{\ell-1}) = \sum_{\ell=0}^\infty f_\ell,
\]
which converges in $L^q(\Omega)$ by (\ref{temp1}). Furthermore, for $\ell=0$,
\[
\|f_0\|_{L^q(\Omega)} \le \|f-g_0\|_{L^q(\Omega)} + \|f\|_{L^q(\Omega)} \lesssim \|f\|_{\cB_q^\Bs(\Omega)}.
\]
For $\ell \in \bN$, 
\begin{align*}
\|f_\ell\|_{L^q(\Omega)} &= \|g_\ell - f+f- g_{\ell-1}\|_{L^q(\Omega)} \\
&\le \|f-g_\ell\|_{L^q(\Omega)} + \|f - g_{\ell-1}\|_{L^q(\Omega)} \\
&\lesssim b^{-\tilde{s} \ell} \|f\|_{\cB_q^\Bs(\Omega)} + b^{-\tilde{s} (\ell-1)} \|f\|_{\cB_q^\Bs(\Omega)} \\
&\lesssim b^{-\tilde{s} \ell} \|f\|_{\cB_q^\Bs(\Omega)}.
\end{align*}
By decomposing $f_\ell\in \cP_{\ell,\Bk-\boldsymbol{1}}$ in the basis $\rho_{\ell,\Bi}^\Bgamma$, we observe that
\begin{align*}
\|f_\ell\|_{L^q(\Omega)}^q &= \sum_{\Bi\in I_\ell} \int_{\Omega_{\ell,\Bi}} \left| \sum_{\Bgamma \le \Bk -\boldsymbol{1}} a_{\ell,\Bi,\Bgamma} \rho_{\ell,\Bi}^\Bgamma(\Bx) \right|^q d\Bx \\
&= \sum_{\Bi\in I_\ell} b^{-|\Bll|} \int_{\Omega} \left| \sum_{\Bgamma \le \Bk -\boldsymbol{1}} a_{\ell,\Bi,\Bgamma} \Bx^{\Bgamma} \right|^q d\Bx \\
&\asymp b^{-\ell} \|\Ba_\ell\|_q^q,
\end{align*}
where we use $b^{-|\Bll|} \asymp b^{-\ell}$ and the fact that all norms are equivalent in finite dimensional spaces in the last inequality. Hence,
\[
\|\Ba_\ell\|_q \lesssim b^{\ell/q} \|f_\ell\|_{L^q(\Omega)} \lesssim b^{(1/q-\tilde{s})\ell} \|f\|_{\cB_q^\Bs(\Omega)}.
\]
Similarly, for $p\ge q$,
\begin{align*}
\|f_\ell\|_{L^p(\Omega)} &\asymp b^{-\ell/p} \|\Ba_\ell\|_p \le b^{-\ell/p} \|\Ba_\ell\|_q \\
&\lesssim b^{(1/q-1/p-\tilde{s})\ell} \|f\|_{\cB_q^\Bs(\Omega)}.
\end{align*}
Since $\tilde{s}>1/q-1/p$, we conclude that the sum $\sum_{\ell=0}^\infty f_\ell$ converges to some function $\widetilde{f}$ in $L^p(\Omega)$. However, we already know that it converges to $f$ in $L^q(\Omega)$. Hence, we must have $f=\widetilde{f}$ almost everywhere.

When $q=\infty$, we can replace (\ref{temp1}) by 
\begin{align*}
\|f-g_\ell\|_{L^\infty(\Omega)} &\le \max_{\Bi \in I_\ell} \| f- g_\ell\|_{L^\infty(\Omega_{\ell,\Bi})} \lesssim \max_{\Bi \in I_\ell} \sum_{j=1}^d \omega_{k_j}^j(f,b^{-\ell_j},\Omega_{\ell,\Bi})_\infty \\
&\le \sum_{j=1}^d b^{-s_j \ell_j} |f|_{\cB^{\Bs,j}_{\infty,\infty}(\Omega)} \lesssim b^{-\tilde{s} \ell} \|f\|_{\cB_\infty^\Bs(\Omega)}.
\end{align*}
Other estimates can be proven similarly as above.
\end{proof}

\subsubsection{Proof of Theorem \ref{app of anisotropic}}

Now, we can combine Lemmas \ref{app of poly} and \ref{anisotropic decomp} to approximate a target function in anisotropic Besov spaces and prove Theorem \ref{app of anisotropic}. We will first derive approximation result by removing an arbitrarily small trifling set in the following proposition and then show that this restriction can be removed. Recall that we denote $\Bll =(\ell_1,\dots, \ell_d)$ with $\ell_j = \lfloor \ell \tilde{s}/s_j \rfloor$ for $\ell\in \bN_0$.

\begin{proposition}\label{app of anisotropic part}
Let $1\le q\le p\le \infty$ and $\Bs\in (0,\infty)^d$. Assume that $\tilde{s}>1/q-1/p$. Suppose $\alpha,\beta \in \bN$, $\ell^*$ is the largest integer such that $|\Bll^*| = \sum_{j=1}^d \lfloor \ell^* \tilde{s}/s_j \rfloor \le 2\kappa (\alpha+\beta)$ with
\[
\kappa := \frac{\tilde{s}}{\tilde{s}+1/p-1/q} \in [1,\infty).
\]
Then, for any $f\in \cB_q^\Bs(\Omega)$, there exists a network $g_{\alpha,\beta} \in \NN(W,L)$ with $W \le Cb^{\alpha}$ and $L\le Cb^{\beta}$ such that, for any $0<\epsilon<\min_{j\in [1:d]} b^{-\ell_j^*}$,
\[
\|f-g_{\alpha,\beta} \|_{L^p(\Omega_{\ell^*,\epsilon})} \le C b^{-2\tilde{s}(\alpha+\beta)} \|f\|_{\cB_q^\Bs(\Omega)}.
\]
Here the constants $C:=C(p,q,\Bs,d,b)$ depend only on $p,q,\Bs,d$ and the base $b$.
\end{proposition}
\begin{proof}
We use similar argument as \citet[Proposition 4.13]{yang2025optimal}. Without loss of generality, we assume that $\|f\|_{\cB_q^\Bs(\Omega)}\le 1$. By Lemma \ref{anisotropic decomp}, we can decompose $f\in \cB_q^\Bs(\Omega)$ as
\[
f = \sum_{\ell=0}^\infty f_\ell,\quad f_\ell(\Bx) = \sum_{\Bi \in I_\ell, \Bgamma \le \Bk -\boldsymbol{1}} a_{\ell,\Bi,\Bgamma} \rho_{\ell,\Bi}^\Bgamma(\Bx) \in \cP_{\ell,\Bk-\boldsymbol{1}},
\]
where $k_j=\lfloor s_j \rfloor +1$, $f_\ell$ and its coefficients $\Ba_\ell = (a_{\ell,\Bi,\Bgamma})_{\Bi \in I_\ell, \Bgamma \le \Bk -\boldsymbol{1}}$ satisfy
\begin{align*}
\|\Ba_\ell\|_q &\lesssim b^{(1/q-\tilde{s})\ell} \lesssim b^{(1/q-\tilde{s})|\Bll|},\\
\|f_\ell\|_{L^p(\Omega)} &\lesssim b^{(1/q-1/p-\tilde{s})\ell} \lesssim b^{(1/q-1/p-\tilde{s})|\Bll|},
\end{align*}
where we use $0\le \ell-|\Bll|\le d$ and the implied constants depend on $p,q,\Bs,d,b$.
We are going to apply Lemma \ref{app of poly} to approximate each $f_\ell$ by a neural network $g_\ell$ for $\ell \in \cL:= [0:\ell^*]$. To do this, we will need to choose the parameters $\delta, S, T, W_0, L_0$ in Lemma \ref{app of poly} as functions of $\ell$. Firstly, we set
\begin{align*}
W_0(\ell) &= \left\lceil (\alpha/2) \log_2 b \right\rceil,  \\
L_0(\ell) &=  \left\lceil 4(\tilde{s}+\kappa) \beta \right\rceil.
\end{align*}
Since $|\Bll^*| \le 2\kappa(\alpha+\beta)$, we have the following estimates on the width in Lemma \ref{app of poly},
\begin{align*}
\sum_{j=1}^d b^{\ell_j/L_0 } &\le d b^{|\Bll^*|/L_0} \le d b^{2\kappa(\alpha+\beta)/(4\kappa \beta)} \le C b^{\alpha/2}, \\
W_02^{W_0} &\le C \alpha b^{\alpha/2}.
\end{align*}
Using the inequalities $\alpha \beta \ge (\alpha+\beta)/2$ for $\alpha,\beta \ge 1$ and $|\Bll^*|\le 2\kappa(\alpha+\beta)$, we get
\begin{equation}\label{small error}
4^{-W_0L_0} \le b^{-\alpha L_0} \le b^{-4(\tilde{s}+\kappa )\alpha \beta} \le b^{-2(\tilde{s}+\kappa )(\alpha + \beta)} \le b^{-2\tilde{s}(\alpha+\beta) - |\Bll^*|}.
\end{equation}
To choose the remaining parameters $S, T$ and $\delta$, we are going to decompose the index set $\cL$ into two subsets according to the two regimes in Lemma \ref{app of poly}. Since the desired network $g_{\alpha,\beta}=\sum_{\ell=0}^{\ell^*} g_\ell$ is a summation of small networks. In order to maintain the size of $g_{\alpha,\beta}$, we will need to further decompose each index subset above into two parts, so that we can apply Proposition \ref{basic constr} Part (4) in two different ways (through parallel or sequential computation). Overall, we decompose the index set $\cL$ into four disjoint subsets $\cL = \cup_{i=1}^4 \cL_i$, which will be given explicitly below. For each $\cL_i$, we denote
\[
F_i:= \sum_{\ell \in \cL_i} f_\ell,\quad \mbox{and}\quad G_i:= \sum_{\ell \in \cL_i} g_\ell.
\]
By choosing $S(\ell), T(\ell)$ and $\delta(\ell)$ appropriately in Lemma \ref{app of poly}, we are going to show that $G_i \in \NN(W_i,L_i)$ with width $W_i\le Cb^{\alpha}$ and depth $L_i\le Cb^{\beta}$ and derive the approximation error bound $\|F_i-G_i \|_{L^p(\Omega_{\ell^*,\epsilon})} \le C b^{-2\tilde{s}(\alpha+\beta)}$ for each $i=[1:4]$. Once we establish these estimates, we can construct the desired network as 
\[
g_{\alpha,\beta} = \sum_{i=1}^4 G_i =\sum_{\ell=0}^{\ell^*} g_\ell \in \NN(W,L),
\]
where $W\le Cb^{\alpha}$ and $L\le Cb^{\beta}$ by Proposition \ref{basic constr}. Its approximation error can be bounded as
\begin{align*}
\|f-g_{\alpha,\beta} \|_{L^p(\Omega_{\ell^*,\epsilon})} &\le \sum_{i=1}^4 \|  F_i - G_i\|_{L^p(\Omega_{\ell^*,\epsilon})} + \sum_{\ell=\ell^*+1}^\infty \|f_\ell\|_{L^p(\Omega_{\ell^*,\epsilon})} \\
&\le C b^{-2\tilde{s}(\alpha+\beta)} + C \sum_{\ell=\ell^*+1}^\infty  b^{(1/q-1/p-\tilde{s})|\Bll|} \\
&\le C b^{-2\tilde{s}(\alpha+\beta)} + C b^{2\kappa(\alpha+\beta)(1/q-1/p-\tilde{s})}  \\
&\le C b^{-2\tilde{s}(\alpha+\beta)},
\end{align*}
where we use $|\Bll^*|+1\ge 2\kappa(\alpha+\beta)$ in the second last inequality and $\kappa(1/q-1/p-\tilde{s})=-\tilde{s}$ in the last inequality.

Now, we give the construction of $G_i$ below. Recall that $\Bll =(\ell_1,\dots, \ell_d)$ with $\ell_j = \lfloor \ell \tilde{s}/s_j \rfloor$.

\textbf{Case 1:} $\ell \in \cL_1=\{\ell \in \cL: |\Bll| \le 2\beta\}$. We choose
\[
\delta(\ell) = |\Bll|/q + (\tilde{s}+1)(2\alpha+2\beta-|\Bll|).
\]
Since $\delta q \ge |\Bll|$, we apply Lemma \ref{app of poly} Part (2) to approximate $f_\ell$ with parameters
\begin{align*}
S(\ell) &= \left\lceil b^{\delta+|\Bll|/2-|\Bll|/q} \right\rceil, \\
T(\ell) &= \left\lceil (\tilde{s}+1)(2+2\beta-|\Bll|) \right\rceil.
\end{align*}
This gives us a neural network $g_\ell \in \NN(W(\ell),L(\ell))$ with width
\begin{align*}
W(\ell) &\le C \max \left\{ \frac{b^{\delta+|\Bll|/2-|\Bll|/q}}{S \sqrt{T}} , b^{(\delta-|\Bll|/q)/T}, \sum_{j=1}^d b^{\ell_j/L_0 }, W_0 2^{W_0} \right\} \\
&\le C \max \left\{ 1, b^{\alpha+1}, b^{\alpha/2}, \alpha b^{\alpha/2} \right\} \\
&\le Cb^{\alpha},
\end{align*}
and depth
\[
L(\ell) \le C \left(\left\lceil b^{-\delta + |\Bll|/q} S \right\rceil T+L_0 \right) \le C \left(b^{|\Bll|/2}(2+2\beta-|\Bll|) + \beta \right).
\]
Since $\|\Ba_\ell\|_q \lesssim b^{(1/q-\tilde{s})|\Bll|}$, its approximation error is
\begin{equation}\label{error 1}
\begin{aligned}
\|f_\ell-g_\ell\|_{L^p(\Omega_{\ell,\epsilon})} &\le C \left(b^{-\delta} + 4^{-W_0L_0}\right) b^{(1/q-\tilde{s})|\Bll|} \\
&\le C b^{-2\tilde{s}(\alpha+\beta)} \left( b^{-2\alpha-2\beta+|\Bll|} +b^{-\tilde{s}|\Bll|} \right),
\end{aligned}
\end{equation}
where we also use the inequality (\ref{small error}). By Proposition \ref{basic constr} Part (4), we can implement the sum $G_1 = \sum_{\ell\in\cL_1} g_\ell$ by a neural network in a sequential way, so that $G_1\in \NN(W_1,L_1)$ with
\begin{align*}
W_1 &= \max_{\ell \in \cL_1} W(\ell)+2d+2 \le Cb^{\alpha}, \\
L_1 &= \sum_{\ell \in \cL_1} L(\ell) \le C \sum_{|\Bll| \le 2\beta} \left(b^{|\Bll|/2}(2+2\beta-|\Bll|) + \beta \right) \le C b^{\beta},
\end{align*}
because the sum is bounded by convergent geometric series. Since $\Omega_{\ell^*,\epsilon} \subseteq \Omega_{\ell,\epsilon}$ for $\ell \le \ell^*$, we get the approximation error bound
\begin{align*}
\|F_1-G_1\|_{L^p(\Omega_{\ell^*,\epsilon})} &\le \sum_{|\Bll| \le 2\beta} \|f_\ell-g_\ell\|_{L^p(\Omega_{\ell^*,\epsilon})} \\
&\le C b^{-2\tilde{s}(\alpha+\beta)} \sum_{|\Bll| \le 2\beta} \left( b^{-2\alpha-2\beta+|\Bll|} +b^{-\tilde{s}|\Bll|} \right) \\
&\le C b^{-2\tilde{s}(\alpha+\beta)}.
\end{align*}

\textbf{Case 2:} $\ell \in \cL_2=\{\ell \in \cL: 2\beta < |\Bll| \le 2\alpha + 2\beta\}$. We choose $\delta(\ell)$ as in Case 1 so that $\delta q \ge |\Bll|$ and we can apply Lemma \ref{app of poly} Part (2) with parameters
\begin{align*}
S(\ell) &= \left\lceil b^{\beta +\delta-|\Bll|/q} \right\rceil, \\
T(\ell) &= \left\lceil 2(\tilde{s}+1) \right\rceil.
\end{align*}
Thus, we approximate $f_\ell$ by a neural network $g_\ell \in \NN(W(\ell),L(\ell))$ with width
\begin{align*}
W(\ell) &\le C \max \left\{ \frac{b^{\delta+|\Bll|/2-|\Bll|/q}}{S \sqrt{T}} , b^{(\delta-|\Bll|/q)/T}, \sum_{j=1}^d b^{\ell_j/L_0 }, W_0 2^{W_0} \right\} \\
&\le C \max \left\{ b^{|\Bll|/2-\beta}, b^{\alpha+\beta-|\Bll|/2}, b^{\alpha/2}, \alpha b^{\alpha/2} \right\} \\
&\le Cb^{\alpha} \left( b^{|\Bll|/2-\alpha-\beta} + b^{\beta - |\Bll|/2} + \alpha b^{-\alpha/2} \right),
\end{align*}
and depth
\[
L(\ell) \le C \left(\left\lceil b^{-\delta + |\Bll|/q} S \right\rceil T+L_0 \right) \le C \left(b^{\beta} + \beta \right) \le Cb^{\beta}.
\]
Using Proposition \ref{basic constr} Part (4), we can implement the sum $G_2 = \sum_{\ell\in\cL_2} g_\ell$ in a parallel way, so that $G_2\in \NN(W_2,L_2)$ with
\begin{align*}
W_2 &= \sum_{\ell \in \cL_2} W(\ell) \le Cb^{\alpha} \sum_{2\beta < |\Bll| \le 2\alpha + 2\beta} \left( b^{|\Bll|/2-\alpha-\beta} + b^{\beta - |\Bll|/2} + \alpha b^{-\alpha/2} \right) \le Cb^{\alpha}, \\
L_2 &= \max_{\ell \in \cL_2} L(\ell) \le  C b^{\beta}.
\end{align*}
By Lemma \ref{app of poly}, we also get the approximation error (\ref{error 1}) in this case. Hence,
\begin{align*}
\|F_2-G_2\|_{L^p(\Omega_{\ell^*,\epsilon})} &\le \sum_{2\beta < |\Bll| \le 2\alpha + 2\beta} \|f_\ell-g_\ell\|_{L^p(\Omega_{\ell^*,\epsilon})} \\
&\le C b^{-2\tilde{s}(\alpha+\beta)} \sum_{2\beta < |\Bll| \le 2\alpha + 2\beta} \left( b^{-2\alpha-2\beta+|\Bll|} +b^{-\tilde{s}|\Bll|} \right) \\
&\le C b^{-2\tilde{s}(\alpha+\beta)}.
\end{align*}

\textbf{Case 3:} $\ell \in \cL_3=\{\ell \in \cL: 2\alpha+2\beta < |\Bll| \le 2\alpha+2\beta + 2\alpha/\tau\}$, where $\tau$ is chosen to satisfy 
\begin{equation}\label{tau}
0< \tau < \frac{\tilde{s}+1/p-1/q}{1/q-1/p} = \frac{1}{\kappa -1}.
\end{equation}
Note that this condition can be satisfied since $q\le p$ and $\tilde{s}>1/q-1/p$ by assumption. In this case, we choose
\begin{align*}
\delta(\ell) &= (2\alpha+2\beta)/q - \tau (|\Bll| - 2\alpha- 2\beta)/q, \\
S(\ell) &= b^{\beta}, \\
T(\ell) &= \left\lceil 2/\tau +2 \right\rceil,
\end{align*}
Since $\delta q \le (2\alpha+2\beta) \le |\Bll|$, we can apply Lemma \ref{app of poly} Part (1) to approximate $f_\ell$ by a neural network $g_\ell \in \NN(W(\ell),L(\ell))$ with depth
\[
L(\ell) \le C(ST+L_0) \le C \left(b^{\beta} + \beta \right) \le Cb^{\beta}.
\]
Since
\[
\frac{|\Bll| - \delta q}{T} \le \frac{(1+\tau)(|\Bll| - 2\alpha-2\beta)}{2/\tau + 2} = \tau(|\Bll|/2-\alpha-\beta),
\]
the width satisfies
\begin{align*}
W(\ell) &\le C \max \left\{ \frac{b^{\delta q/2}}{S \sqrt{T}}, b^{(|\Bll| -\delta q)/T}, \sum_{j=1}^d b^{\ell_j/L_0 }, W_0 2^{W_0} \right\} \\
&\le C \max \left\{ b^{\alpha+\tau(\alpha+\beta-|\Bll|/2)}, b^{\tau(|\Bll|/2-\alpha-\beta)}, b^{\alpha/2}, \alpha b^{\alpha/2} \right\} \\
&\le C b^{\alpha}\left( b^{\tau(\alpha+\beta-|\Bll|/2)} + b^{-\alpha+\tau(|\Bll|/2-\alpha-\beta)} + \alpha b^{-\alpha/2} \right).
\end{align*}
By Proposition \ref{basic constr} Part (4), we can implement the sum $G_3 = \sum_{\ell\in\cL_3} g_\ell$ in a parallel way, so that $G_3\in \NN(W_3,L_3)$ with
\begin{align*}
W_3 &= \sum_{\ell \in \cL_3} W(\ell) \le C b^{\alpha} \sum_{\alpha+\beta < |\Bll|/2 \le \alpha+\beta + \alpha/\tau} \left( b^{\tau(\alpha+\beta-|\Bll|/2)} + b^{-\alpha+\tau(|\Bll|/2-\alpha-\beta)} + \alpha b^{-\alpha/2} \right) \le Cb^{\alpha}, \\
L_3 &= \max_{\ell \in \cL_3} L(\ell) \le  C b^{\beta}.
\end{align*}
By Lemma \ref{app of poly} and $\|\Ba_\ell\|_q \lesssim b^{(1/q-\tilde{s})|\Bll|}$, the approximation error satisfies
\begin{equation}\label{error 2}
\begin{aligned}
\|f_\ell-g_\ell\|_{L^p(\Omega_{\ell,\epsilon})} &\le C \left(b^{\delta q/p -|\Bll|/p-\delta} + 4^{-W_0L_0}\right) b^{(1/q-\tilde{s})|\Bll|} \\
&\le C b^{-2\tilde{s}(\alpha+\beta)} \left( b^{\eta (|\Bll| -2\alpha-2\beta)} +b^{-\tilde{s}|\Bll|} \right),
\end{aligned}
\end{equation}
where $\eta:= 1/q-1/p-\tilde{s}+\tau(1/q-1/p)<0$ by condition (\ref{tau}) and we use (\ref{small error}) in the last inequality. Hence,
\begin{align*}
\|F_3-G_3\|_{L^p(\Omega_{\ell^*,\epsilon})} &\le \sum_{2\alpha+2\beta < |\Bll| \le 2\alpha+2\beta + 2\alpha/\tau} \|f_\ell-g_\ell\|_{L^p(\Omega_{\ell^*,\epsilon})} \\
&\le C b^{-2\tilde{s}(\alpha+\beta)} \sum_{2\alpha+2\beta < |\Bll| \le 2\alpha+2\beta + 2\alpha/\tau} \left( b^{\eta (|\Bll| -2\alpha-2\beta)} +b^{-\tilde{s}|\Bll|} \right) \\
&\le C b^{-2\tilde{s}(\alpha+\beta)}.
\end{align*}

\textbf{Case 4:} $\ell \in \cL_4=\{\ell \in \cL: 2\alpha+2\beta + 2\alpha/\tau < |\Bll| \le 2\kappa (\alpha+\beta) \}$. Observe that, by condition (\ref{tau}),
\[
|\Bll^*| - 2\alpha -2\beta \le 2(\kappa-1)(\alpha+\beta)<2(\alpha+\beta)/\tau.
\]
Hence, we can choose $\delta(\ell)$ as in Case 3 such that $\delta>0$ and $\delta q < |\Bll|$. We apply Lemma \ref{app of poly} Part (1) to approximate $f_\ell$ with parameters
\begin{align*}
S(\ell) &= \left\lceil b^{\delta q/2} \right\rceil, \\
T(\ell) &= \left\lceil 2/\tau + 2 \right\rceil \left\lceil \tau(|\Bll|/2-\alpha-\beta)-\alpha \right\rceil,
\end{align*}
which gives us a neural network $g_\ell \in \NN(W(\ell),L(\ell))$ with depth
\[
L(\ell) \le C(ST+L_0) \le C \left(b^{\alpha+\beta-\tau(|\Bll|/2-\alpha-\beta)}(\tau(|\Bll|/2-\alpha-\beta)-\alpha) + \beta \right).
\]
Since
\begin{align*}
\frac{|\Bll| -\delta q}{T} &\le \frac{(1+\tau)(|\Bll| - 2\alpha-2\beta)}{(2/\tau + 2) \lceil \tau(|\Bll|/2-\alpha-\beta)-\alpha \rceil} = \frac{\tau(|\Bll|/2-\alpha-\beta)}{\lceil \tau(|\Bll|/2-\alpha-\beta)-\alpha \rceil} \\
&\le 1+ \frac{\alpha}{\lceil \tau(|\Bll|/2-\alpha-\beta)-\alpha \rceil} \le 1+\alpha,
\end{align*}
its width $W(\ell)$ can be bounded as
\begin{align*}
W(\ell) &\le C \max \left\{ \frac{b^{\delta q/2}}{S \sqrt{T}}, b^{(|\Bll| -\delta q)/T}, \sum_{j=1}^d b^{\ell_j/L_0 }, W_0 2^{W_0} \right\} \\
&\le C \max \left\{ 1, b^{\alpha+1}, b^{\alpha/2}, \alpha b^{\alpha/2} \right\} \\
&\le Cb^{\alpha}.
\end{align*}
Using Proposition \ref{basic constr} Part (4), we can implement the sum $G_4 = \sum_{\ell\in\cL_4} g_\ell$ in a sequential way, so that $G_4\in \NN(W_4,L_4)$ with
\begin{align*}
W_4 &= \max_{\ell \in \cL_4} W(\ell) +2d+2 \le Cb^{\alpha}, \\
L_4 &= \sum_{\ell \in \cL_4} L(\ell) \le C \sum_{2\alpha+2\beta + 2\alpha/\tau < |\Bll| \le 2\kappa (\alpha+\beta)} \left(b^{\alpha+\beta-\tau(|\Bll|/2-\alpha-\beta)}(\tau(|\Bll|/2-\alpha-\beta)-\alpha) + \beta \right) \\
&\le C \sum_{j = \lceil 2\alpha/\tau \rceil}^{\lfloor 2(\alpha+\beta)/\tau \rfloor} \left(b^{\alpha+\beta-\tau j/2}(\tau j/2-\alpha) + \beta \right) \le Cb^{\beta},
\end{align*}
where we use $2\kappa (\alpha+\beta) - 2\alpha -2\beta <2(\alpha+\beta)/\tau$ and the last series is dominated by the term corresponding to $j=\lceil 2\alpha/\tau \rceil$. By Lemma \ref{app of poly}, we also get the approximation error bound (\ref{error 2}) in this case. Hence,
\begin{align*}
\|F_4-G_4\|_{L^p(\Omega_{\ell^*,\epsilon})} &\le \sum_{2\alpha+2\beta + 2\alpha/\tau < |\Bll| \le 2\kappa (\alpha+\beta)} \|f_\ell-g_\ell\|_{L^p(\Omega_{\ell^*,\epsilon})} \\
&\le C b^{-2\tilde{s}(\alpha+\beta)} \sum_{2\alpha+2\beta + 2\alpha/\tau < |\Bll| \le 2\kappa (\alpha+\beta)} \left( b^{\eta (|\Bll| -2\alpha-2\beta)} +b^{-\tilde{s}|\Bll|} \right) \\
&\le C b^{-2\tilde{s}(\alpha+\beta)},
\end{align*}
which completes the proof.
\end{proof}

Finally, we show how to remove the trifling region to give a proof of Theorem \ref{app of anisotropic}. Following the idea from \citet{siegel2023optimal}, we will use different bases $b_k$ to create minimally overlapping trifling regions and apply the sorting network in Proposition \ref{select median} to select a good approximation.

\begin{proof}[Proof of Theorem \ref{app of anisotropic}]
Observe that when $p<q$, we have $\|f-g\|_{L^p(\Omega)} \le \|f-g\|_{L^q(\Omega)}$ with $\Omega = [0,1)^d$ as above. Hence, it is enough to consider the case $1\le q\le p\le \infty$. Without loss of generality, we assume that the target function is normalized so that $\|f\|_{\cB^\Bs_q(\Omega)} \le 1$. We will make use of different bases $b$ to remove the trifling region in Proposition \ref{app of anisotropic part}. To do this, we let $r$ be the smallest integer such that $2^r \ge 2d+2$ and let $b_k$ be the $k$-th prime number for $k\in[1:2^r]$. Note that $r$ and $b_k$ only depend on the dimension $d$. Then, it is sufficient to show that, for any integers $m,n \ge b_{2^r}$, there exists $g\in \NN(W,L)$ with $W\le Cm$ and $L\le Cn$ such that
\[
\|f-g\|_{L^p(\Omega)} \le C (mn)^{-2\tilde{s}}.
\]

For $k\in[1:2^r]$, we denote $\alpha_k = \lfloor \log_{b_k} m \rfloor$ and $\beta_k = \lfloor \log_{b_k} n \rfloor$. As in Proposition \ref{app of anisotropic part}, we let $\ell_k^*$ be the largest integer such that $|\Bll_k^*| \le 2\kappa (\alpha_k+\beta_k)$, where $\kappa = \tilde{s}/(\tilde{s}+1/p-1/q)$ and $\Bll_k^* = (\ell_{k,1}^*,\dots,\ell_{k,d}^*)$ with $\ell_{k,j}^* = \lfloor \ell_k^* \tilde{s}/s_j \rfloor$. Since the bases $b_k$ are all pairwise relatively prime, we observe that, for each $j\in [1:d]$, the following set of numbers are all distinct
\[
\cA_j:= \bigcup_{k=1}^{2^r} \left\{ \frac{1}{b_k^{\ell_{k,j}^*}},\dots, \frac{b_k^{\ell_{k,j}^*}-1}{b_k^{\ell_{k,j}^*}} \right\}.
\]
We choose $\epsilon>0$ small enough so that 
\[
\epsilon < \min_{j\in [1:d]} \min_{u\neq v\in \cA_j} |u-v|.
\]
With this choice of $\epsilon$, we have, for any $j\in [1:d]$, any element $x_j\in [0,1)$ is contained in at most one of the sets
\begin{equation}\label{bad sets}
\left[i b_k^{-\ell_{k,j}^*}-\epsilon, i b_k^{-\ell_{k,j}^*}\right), \quad i\in \left[1: b_k^{\ell_{k,j}^*} -1\right] \mbox{ and } k\in [1:2^r].
\end{equation}
Now, we let $\Omega_{\Bll_k^*,\epsilon}$ be the good region (\ref{good region}) with resolution $\Bll_k^*$ and base $b_k$ for $k\in [1:2^r]$. Then, any $\Bx\in \Omega$ is contained in at least $2^r-d$ good regions $\Omega_{\Bll_k^*,\epsilon}$, because each coordinate of $\Bx$ can be contained in at most one set (\ref{bad sets}) from the trifling regions. In other words, if we denote, for any $\Bx\in \Omega$,
\[
\cK(\Bx) := \left\{k \in [1:2^r]: \Bx\in \Omega_{\Bll_k^*,\epsilon}\right\},
\]
then $\cK(\Bx)$ has at least $2^r-d \ge 2^{r-1}+1$ elements, since $2^r \ge 2d+2$ by the choice of $r$.

Using Proposition \ref{app of anisotropic part} with parameters $\alpha_k, \beta_k$, the base $b_k$ and the above $\epsilon$ for each $k\in [1:2^r]$, we get a neural network $g_k \in \NN(W_k,L_k)$ with $W_k \le Cb_k^{\alpha_k}\le Cm$ and $L_k\le Cb_k^{\beta_k} \le Cn$ such that
\[
\| f- g_k \|_{L^p(\Omega_{\Bll_k^*,\epsilon})} \le C b_k^{-2\tilde{s}(\alpha_k + \beta_k)} \le C(mn)^{-2\tilde{s}},
\]
because $b_k^{-1} m \le b_k^{\alpha_k} \le m$ and $b_k^{-1} n \le b_k^{\beta_k} \le n$ due to the choice of $\alpha_k$ and $\beta_k$. Next, we let $\psi_{2^{r-1}} \in \NN(2^{r+2},r(r+1)/2)$ be the network in Proposition \ref{select median} that selects the $2^{r-1}$-largest value from $2^r$ values. We construct the desired network $g$ as
\[
\Bx \to 
\begin{pmatrix}
g_1(\Bx) \\
\vdots \\
g_{2^r}(\Bx)
\end{pmatrix}
\to \psi_{2^{r-1}}(g_1(\Bx),\dots, g_{2^r}(\Bx)).
\]
Then, by Proposition \ref{basic constr}, it is easy to see that $g\in \NN(W,L)$ with 
\begin{align*}
W &\le \max \left\{ \sum_{k=1}^{2^r} W_k, 2^{r+2} \right\} \le C m, \\
L &\le \max_{1\le k\le 2^r} L_k + r(r+1)/2 \le C n.
\end{align*}

Finally, we estimate the approximation error. For each $\Bx\in\Omega$, since $|\cK(\Bx)| \ge 2^{r-1}+1$, the $2^{r-1}$-largest element of $\{g_1(\Bx),\dots,g_{2^r}(\Bx)\}$ must be both larger and smaller than some element of $\{g_k(\Bx):k\in \cK(\Bx)\}$. In other words, 
\[
\min_{k\in \cK(\Bx)} g_k(\Bx) \le g(\Bx) = \psi_{2^{r-1}}(g_1(\Bx),\dots, g_{2^r}(\Bx)) \le \max_{k\in \cK(\Bx)} g_k(\Bx),
\]
from which we conclude that
\[
|f(\Bx) -g(\Bx)| \le \max_{k\in \cK(\Bx)} |f(\Bx) - g_k(\Bx)|. 
\]
When $p=\infty$, we finish the proof by noticing that the right hand side is bounded by $C(mn)^{-2\tilde{s}}$ by the definition of $\cK(\Bx)$. When $p<\infty$, we have
\begin{align*}
\int_{\Omega} |f(\Bx) -g(\Bx)|^p d\Bx &\le \int_{\Omega} \sum_{k\in \cK(\Bx)} |f(\Bx) - g_k(\Bx)|^p d\Bx \le \sum_{k=1}^{2^r} \| f- g_k \|_{L^p(\Omega_{\Bll_k^*,\epsilon})}^p \\
&\le C (mn)^{-2\tilde{s}p},
\end{align*}
which completes the proof by taking $p$-th roots.
\end{proof}

\subsection{Approximation of mixed smooth Besov functions}\label{sec: app of mixed}

In this section, we study how well mixed smooth Besov functions can be approximated by ReLU neural networks. The analysis is similar to Section \ref{sec: app of anisotropic}. The main difference is that we will need to make use of the full generality of the anisotropic partition (\ref{partition}) rather than restricting to a specific grid as (\ref{grid for anisotropic}).

\subsubsection{Approximation by piecewise polynomials}

As in Lemma \ref{anisotropic decomp}, we first decompose mixed smooth Besov functions into piecewise polynomials with different resolutions and estimate the norm of coefficients of these piecewise polynomials so that we can apply Lemma \ref{app of poly} later. 

\begin{lemma}\label{mixed decomp}
Let $1\le q\le p\le \infty$, $s\in (0,\infty)$ and $k=\lfloor s \rfloor +1$. If $s>1/q-1/p$, then we can decompose any $f\in \MB_q^s(\Omega)$ as 
\[
f = \sum_{\Bll\in \bN_0^d} f_\Bll,\quad f_\Bll(\Bx) = \sum_{\Bi \in I_\Bll, \Bgamma \le k -1} a_{\Bll,\Bi,\Bgamma} \rho_{\Bll,\Bi}^\Bgamma(\Bx) \in \cP_{\Bll,k-1},
\]
where $\rho_{\Bll,\Bi}^\Bgamma$ is the basis (\ref{basis}), $a_{\Bll,\Bi,\Bgamma}\in \bR$ and the convergence is in $L^p(\Omega)$. Furthermore, if we let $\Ba_\Bll = (a_{\Bll,\Bi,\Bgamma})_{\Bi \in I_\Bll, \Bgamma \le k-1}$ be the vector of coefficients of $f_\Bll$, then we have the following estimates
\begin{align*}
\|\Ba_\Bll\|_q &\lesssim b^{(1/q-s)|\Bll|} \|f\|_{\MB_q^s(\Omega)},\\
\|f_\Bll\|_{L^p(\Omega)} &\lesssim b^{(1/q-1/p-s)|\Bll|} \|f\|_{\MB_q^s(\Omega)},
\end{align*}
where the implied constants depend only on $p, q, s, d$ and the base $b$.
\end{lemma}
\begin{proof}
We are going to decompose $f\in L^q(\Omega)$ by using Smolyak-type algorithms \citep{andrianov1997two,dung2018hyperbolic}. For $\Bll=(\ell_1,\dots,\ell_d) \in \bN_0^d$, let $P_{\ell_j}^{(j)}f(\Bx)$ be the projection of $f$, as a function of $x_j$ with other components $x_i$ ($i\neq j$) fixed, onto the space of piecewise polynomials $\cP_{\ell_j,k-1}$ on $[0,1]$. We use $\|f\|_{q,j}$ to denote the $L^q$-norm of $f$ as a function of $x_j$ with other components fixed. Then, by Lemma \ref{whitney} with $d=1$, we have
\begin{equation}\label{temp2}
\left\| f- P_{\ell_j}^{(j)} f \right\|_{q,j} \lesssim \sup_{0\le h_j \le b^{-\ell_j}} \left\|\Delta_{h_j}^{k,j} f \right\|_{q,j}.
\end{equation}
If we define
\[
g_\Bll := \prod_{j=1}^d P_{\ell_j}^{(j)} f \in \cP_{\Bll,k-1},
\]
then one can show that $g_\Bll \to f$ in $L^q(\Omega)$ as $\Bll \to \infty$, since it holds for smooth $f$ and hence for $L^q$ functions by density argument.

Now, we let $T_{\ell_j}^{(j)} = I-P_{\ell_j}^{(j)}$ and define
\[
f_\Bll := \prod_{j=1}^d \left(P_{\ell_j}^{(j)} - P_{\ell_j-1}^{(j)} \right) f = \prod_{j=1}^d \left(T_{\ell_j-1}^{(j)} - T_{\ell_j}^{(j)} \right) f,
\]
where $P_{-1}^{(j)} =0$. It is easy to check that $f_\Bll \in \cP_{\Bll,k-1}$. When $f\in \MB_q^s(\Omega)$, we can estimate the norm of $f_\Bll$ similar to \citet{dung2011b} and \citet[Appendix D.2]{suzuki2019adaptivity}. By applying the estimate (\ref{temp2}) to $T_{\ell_i}^{(i)}f$ for $i\neq j$, we have
\begin{align*}
\left\| \left\| T_{\ell_j}^{(j)} T_{\ell_i}^{(i)} f \right\|_{q,j} \right\|_{q,i} &\lesssim \left\| \sup_{0\le h_j\le b^{-\ell_j}} \left\|\Delta_{h_j}^{k,j} T_{\ell_i}^{(i)} f \right\|_{q,j} \right\|_{q,i} \\
&= \sup_{0\le h_j\le b^{-\ell_j}} \left\| \left\| T_{\ell_i}^{(i)} \Delta_{h_j}^{k,j} f \right\|_{q,i} \right\|_{q,j} \\
&\lesssim \sup_{0\le h_j\le b^{-\ell_j}} \sup_{0\le h_i\le b^{-\ell_i}} \left\| \left\| \Delta_{h_i}^{k,i} \Delta_{h_j}^{k,j} f \right\|_{q,i} \right\|_{q,j}
\end{align*}
where the equality follows from the definition of $\Delta_{h_j}^{k,j}$ and Fubini's theorem. Thus, applying the above argument recursively, we conclude that, for $J\subseteq [1:d]$,
\[
\left\| \prod_{j\in J} T_{\ell_j}^{(j)} f \right\|_{L^q(\Omega)} \lesssim \omega_k^J(f,b^{-\Bll},\Omega)_q,
\]
where $b^{-\Bll}=(b^{-\ell_1},\dots, b^{-\ell_d})$. By calculation, we notice that 
\[
f_\Bll = \prod_{j=1}^d \left(T_{\ell_j-1}^{(j)} - T_{\ell_j}^{(j)} \right) f = \sum_{J\subseteq [1:d]} (-1)^{|J|} \left( \prod_{j\in J} T_{\ell_j}^{(j)} \prod_{j\notin J} T_{\ell_j-1}^{(j)} \right) f.
\]
As a consequence, we have
\[
\| f_\Bll \|_{L^q(\Omega)} \le \sum_{J\subseteq [1:d]} \left\| \left( \prod_{j\in J} T_{\ell_j}^{(j)} \prod_{j\notin J} T_{\ell_j-1}^{(j)} \right) f \right\|_{L^q(\Omega)} \lesssim \sum_{J\subseteq [1:d]} \omega_k^{\tilde{J}}(f,b^{-\Bll^{(J)}},\Omega)_q,
\]
where $\Bll^{(J)}_j=\ell_j$ if $j\in J$ and $\Bll^{(J)}_j=\ell_j-1$ if $j\notin J$, and $\tilde{J} = \{j:\ell_j^{(J)}\ge 0\}$. If we denote $E=\{j:\ell_j>0\}$, then $\tilde{J} \supseteq E$ and, by properties of modulus smoothness,
\[
\| f_\Bll \|_{L^q(\Omega)} \lesssim \sum_{J\supseteq E} \omega_k^{J}(f,b^{-\Bll^{(J)}},\Omega)_q \lesssim \sum_{J\supseteq E} \left(\prod_{j\in J}b^{-s\ell_j}\right)|f|_{\cB^{s,J}_{q,\infty}(\Omega)} \lesssim b^{-s|\Bll|} \|f\|_{\MB_q^s(\Omega)},
\]
where we use the definition of the semi-norm $|f|_{\cB^{s,J}_{q,\infty}(\Omega)}$ in the second inequality. Since the above estimate implies that $\sum_{\Bll\in \bN_0^d} \|f_\Bll\|_{L^q(\Omega)}$ is finite, the sum $\sum_{\Bll\in \bN_0^d} f_\Bll$ converges absolutely in $L^q(\Omega)$. Observing that the partial sum
\[
g_{\Bll'} = \sum_{\Bll\le \Bll'} f_\Bll
\]
converges to $f$, we conclude that $f=\sum_{\Bll\in \bN_0^d} f_\Bll$ in $L^q(\Omega)$.

Next, we estimate $\|\Ba_\Bll\|_q$ and $\|f_\Bll\|_{L^p(\Omega)}$. When $q=p=\infty$,  
\[
\|f_\Bll\|_{L^\infty(\Omega)} = \left\| \sum_{\Bi \in I_\Bll, \Bgamma \le k -1} a_{\Bll,\Bi,\Bgamma} \rho_{\Bll,\Bi}^\Bgamma \right\|_{L^\infty(\Omega)} \asymp \|\Ba_\Bll\|_\infty.
\]
When $q<\infty$, similar to the proof of Lemma \ref{anisotropic decomp}, we have
\[
\|f_\Bll\|_{L^q(\Omega)}^q = \sum_{\Bi\in I_\Bll} b^{-|\Bll|} \int_{\Omega} \left| \sum_{\Bgamma \le k -\boldsymbol{1}} a_{\Bll,\Bi,\Bgamma} \Bx^{\Bgamma} \right|^q d\Bx \asymp b^{-|\Bll|} \|\Ba_\Bll\|_q^q.
\]
Hence, $\|\Ba_\Bll\|_q \lesssim b^{|\Bll|/q} \|f_\ell\|_{L^q(\Omega)} \lesssim b^{(1/q-s)|\Bll|} \|f\|_{\MB_q^s(\Omega)}$. 
Similarly, for $p\ge q$,
\begin{align*}
\|f_\Bll\|_{L^p(\Omega)} &\asymp b^{-|\Bll|/p} \|\Ba_\Bll\|_p \le b^{-|\Bll|/p} \|\Ba_\ell\|_q \\
&\lesssim b^{(1/q-1/p-s)|\Bll|} \|f\|_{\MB_q^s(\Omega)}.
\end{align*}
Since $s>1/q-1/p$, we conclude that the sum $\sum_{\Bll\in \bN_0^d} f_\Bll$ converges absolutely to some function $\widetilde{f}$ in $L^p(\Omega)$. Since we already know that it converges to $f$ in $L^q(\Omega)$, we must have $f=\widetilde{f}$ almost everywhere.
\end{proof}

\subsubsection{Proof of Theorem \ref{app of mixed}}

Similar to the proof of Theorem \ref{app of anisotropic}, we first approximate the target function in mixed smooth Besov spaces on the good region, where an arbitrarily small trifling set is removed, and then show that one can also get good approximation on the trifling region by using the method of \citet{siegel2023optimal}.

\begin{proposition}\label{app of mixed part}
Let $1\le q\le p\le \infty$ and $s>1/q-1/p$. Suppose $\alpha,\beta \in \bN$, $\ell^*$ is the largest integer such that $\ell^* \le 2\kappa (\alpha+\beta)$ with
\[
\kappa := \frac{s}{s+1/p-1/q} \in [1,\infty).
\]
Then, for any $f\in \MB_q^s(\Omega)$, there exists a network $g_{\alpha,\beta} \in \NN(W,L)$ with $W \le C\alpha^{d-1}b^{\alpha}$ and $L\le C\beta^{d-1}b^{\beta}$ such that, for any $0<\epsilon<b^{-\ell^*}$,
\[
\|f-g_{\alpha,\beta} \|_{L^p(\Omega_{\ell^*,\epsilon})} \le C b^{-2s(\alpha+\beta)}(\alpha+\beta)^{d-1} \|f\|_{\MB_q^s(\Omega)},
\]
where $\Omega_{\ell^*,\epsilon} := \cap_{|\Bll|\le \ell^*} \Omega_{\Bll,\epsilon}$. Here the constants $C:=C(p,q,s,d,b)$ depend only on $p,q,s,d$ and the base $b$.
\end{proposition}
\begin{proof}
Similar to the proof of Proposition \ref{app of anisotropic part}, we can assume that $\|f\|_{\MB_q^s(\Omega)}\le 1$ and decompose $f\in \MB_q^s(\Omega)$ by Lemma \ref{mixed decomp} as 
\[
f = \sum_{\Bll\in \bN_0^d} f_\Bll,\quad f_\Bll(\Bx) = \sum_{\Bi \in I_\Bll, \Bgamma \le k -1} a_{\Bll,\Bi,\Bgamma} \rho_{\Bll,\Bi}^\Bgamma(\Bx) \in \cP_{\Bll,k-1},
\]
where $k= \lfloor s \rfloor +1$, $\|\Ba_\Bll\|_q \lesssim b^{(1/q-s)|\Bll|}$ and
$\|f_\Bll\|_{L^p(\Omega)} \lesssim b^{(1/q-1/p-s)|\Bll|}$. Then, we can apply Lemma \ref{app of poly} to approximate each $f_\Bll$ by a neural network $g_\Bll$ for $|\Bll|\le \ell^*$. The construction of $g_\Bll \in \NN(W(|\Bll|),L(|\Bll|))$ can be done as in the proof of Proposition \ref{app of anisotropic part}, since that construction only depends on $|\Bll|$. The main difference is that, for each $\ell\in \bN$, we have at most $d$ vectors $\Bll$ such that $|\Bll|=\ell$ due to the special form of $\Bll$ in Proposition \ref{app of anisotropic part}, while the number of vectors $\Bll\in \bN^d$ such that $|\Bll|=\ell$ in the general case is
\[
N_{\ell,d}:= \binom{\ell +d-1}{d-1} \le \left(\frac{e(\ell +d-1)}{d-1}\right)^{d-1} \lesssim \ell^{d-1},
\]
which depends on $\ell$. We estimate the size of the network $g_{\alpha,\beta} = \sum_{|\Bll|\le \ell^*} g_\Bll$ and its approximation error as follows.

Since $N_{\ell,d}\le N_{\ell^*,d}$ for $\ell \le \ell^*$, it is easy to see that we can decompose the index set $\cL=\{\Bll\in \bN_0^d: |\Bll|\le \ell^*\}$ into $N_{\ell^*,d}$ disjoint subsets $\cup_{i=1}^{N_{\ell^*,d}} \cL_i$ such that for any $\Bll, \Bll'\in \cL_i$, if $\Bll\neq \Bll'$ then $|\Bll|\neq |\Bll'|$. For each $\cL_i$, we denote
\[
F_i = \sum_{\Bll\in \cL_i} f_\Bll,\quad G_i = \sum_{\Bll\in \cL_i} g_\Bll.
\]
By decomposing each $\cL_i$ into at most four disjoint cases as the proof of Proposition \ref{app of anisotropic part} and replacing $\tilde{s}$ by $s$, one can show that $G_i \in \NN(W_i,L_i)$ with $W_i\le Cb^\alpha $ and $L_i\le Cb^\beta$, and its approximation error satisfies
\[
\|F_i-G_i\|_{L^p(\Omega_{\ell^*,\epsilon})} \le C b^{-2s(\alpha+\beta)},
\]
since our construction of the networks and estimate of the approximation errors in Proposition \ref{app of anisotropic part} only depend on $|\Bll|$. Applying Corollary \ref{basi constr co} to the networks $G_i$, $1\le i\le N_{\ell^*,d}$, and noticing that 
\[
N_{\ell^*,d} \lesssim (\ell^*)^{d-1} \lesssim (\alpha+\beta)^{d-1} \lesssim \alpha^{d-1} \beta^{d-1},
\]
we have $g_{\alpha,\beta} = \sum_{i=1}^{N_{\ell^*,d}} G_i \in \NN(W,L)$ with $W \le C\alpha^{d-1}b^{\alpha}$ and $L\le C\beta^{d-1}b^{\beta}$. Finally, by the estimate $\|f_\Bll\|_{L^p(\Omega)} \lesssim b^{(1/q-1/p-s)|\Bll|}$, we can bound the approximation error as
\begin{align*}
\|f-g_{\alpha,\beta} \|_{L^p(\Omega_{\ell^*,\epsilon})} &\le \sum_{i=1}^{N_{\ell^*,d}} \|  F_i - G_i\|_{L^p(\Omega_{\ell^*,\epsilon})} + \sum_{|\Bll|>\ell^*} \|f_\Bll\|_{L^p(\Omega_{\ell^*,\epsilon})} \\
&\le C b^{-2s(\alpha+\beta)} N_{\ell^*,d} + C \sum_{\ell=\ell^*+1}^\infty b^{(1/q-1/p-s)\ell} N_{\ell,d} \\
&\le C b^{-2s(\alpha+\beta)} (\alpha+\beta)^{d-1} + C \sum_{\ell=\ell^*+1}^\infty b^{(1/q-1/p-s)\ell} \ell^{d-1} \\
&\le C b^{-2s(\alpha+\beta)} (\alpha+\beta)^{d-1} + C b^{(\ell^*+1)(1/q-1/p-s)} (\ell^*+1)^{d-1}  \\
&\le C b^{-2s(\alpha+\beta)} (\alpha+\beta)^{d-1},
\end{align*}
where we use $N_{\ell^*,d} \lesssim (\alpha+\beta)^{d-1}$ and $N_{\ell,d}\lesssim \ell^{d-1}$ in the third inequality, and $(\ell^*+1)(1/q-1/p-s) \le 2(\alpha+\beta)\kappa (1/q-1/p-s)= -2(\alpha+\beta)s$ in the last inequality.
\end{proof}

\begin{proof}[Proof of Theorem \ref{app of mixed}]
As in the proof of Theorem \ref{app of anisotropic}, we can assume that $1\le q\le p\le \infty$ and $\|f\|_{\MB^s_q(\Omega)} \le 1$. Since the case $d=1$ is a direct consequence of Theorem \ref{app of anisotropic}, we also assume $d\ge 2$. Let $r$ be the smallest integer such that $2^r \ge 2d+2$ and $b_k$ be the $k$-th prime number for $k\in[1:2^r]$, so that $r$ and $b_k$ only depend on the dimension $d$. It is sufficient to show that, for sufficiently large integers $m$ and $n$, there exists $g\in \NN(W,L)$ with $W\le Cm$ and $L\le Cn$ such that
\begin{equation}\label{temp3}
\|f-g\|_{L^p(\Omega)} \le C (mn)^{-2s}(\log m \log n)^{(d-1)(2s+1)}.
\end{equation}

For each $k\in[1:2^r]$, we let $\alpha_k$ and $\beta_k$ be the largest integers such that $\alpha_k^{d-1}b_k^{\alpha_k} \le m$ and $\beta_k^{d-1}b_k^{\beta_k} \le n$ respectively. Let $\ell_k^*$ be the largest integer such that $\ell_k^* \le 2\kappa (\alpha_k+\beta_k)$, where $\kappa = s/(s+1/p-1/q)$. Notice that the good region $\Omega_{\ell_k^*,\epsilon}$ in Proposition \ref{app of mixed part} can be explicitly written as
\[
\Omega_{\ell_k^*,\epsilon}= \bigcup_{\Bi\in [0:b^{\ell_k^*}-1]} \Omega_{\ell_k^*,\Bi,\epsilon},\quad \Omega_{\ell_k^*,\Bi,\epsilon} = \prod_{j=1}^d
\begin{cases}
[b^{-\ell_k^*} i_j, b^{-\ell_k^*} (i_j+1)-\epsilon), & i_j<b^{\ell_k^*} -1,\\
[1-b^{-\ell_k^*}, 1), & i_j=b^{\ell_k^*} -1.
\end{cases} 
\]
Since the bases $b_k$ are all pairwise relatively prime, the following set of numbers are all distinct
\[
\cA:= \bigcup_{k=1}^{2^r} \left\{ \frac{1}{b_k^{\ell_k^*}},\dots, \frac{b_k^{\ell_k^*}-1}{b_k^{\ell_k^*}} \right\}.
\]
So, for any $\epsilon< \min_{u\neq v\in \cA} |u-v|$, any element of $[0,1)$ is contained in at most one of the sets
\[
\left[i b_k^{-\ell_k^*}-\epsilon, i b_k^{-\ell_k^*}\right), \quad i\in \left[1: b_k^{\ell_k^*} -1\right] \mbox{ and } k\in [1:2^r].
\]
This property implies that, for any $\Bx\in \Omega$, the set $\cK(\Bx) := \{k \in [1:2^r]: \Bx\in \Omega_{\ell_k^*,\epsilon}\}$ has at least $2^r-d \ge 2^{r-1}+1$ elements.

By Proposition \ref{app of mixed part} with parameters $\alpha_k, \beta_k$, the base $b_k$ and the above $\epsilon$ for each $k\in [1:2^r]$, there exists a neural network $g_k \in \NN(W_k,L_k)$ with $W_k \le C\alpha_k^{d-1}b_k^{\alpha_k}\le Cm$ and $L_k\le C\beta_k^{d-1}b_k^{\beta_k} \le Cn$ such that
\begin{align*}
\| f- g_k \|_{L^p(\Omega_{\ell_k^*,\epsilon})} &\le C b_k^{-2s(\alpha_k + \beta_k)}(\alpha_k+\beta_k)^{d-1} \\
&\le C\left(\frac{mn}{(\log m \log n)^{d-1}}\right)^{-2s} (\log m + \log n)^{d-1}\\
&\le C(mn)^{-2s}(\log m \log n)^{(d-1)(2s+1)},
\end{align*}
because $\alpha_k \asymp \log m$, $b_k^{\alpha_k} \asymp m/(\log m)^{d-1}$, $\beta_k \asymp \log n$ and $b_k^{\beta_k} \asymp n/(\log n)^{d-1}$ due to the choice of $\alpha_k$ and $\beta_k$. Let $\psi_{2^{r-1}} \in \NN(2^{r+2},r(r+1)/2)$ be the network in Proposition \ref{select median} that selects the $2^{r-1}$-largest value from $2^r$ values. Then, we can construct the desired function as $g(\Bx) = \psi_{2^{r-1}}(g_1(\Bx),\dots, g_{2^r}(\Bx))$, which is in $\NN(W,L)$ with $W\le Cm$ and $L\le Cn$ by Proposition \ref{basic constr}. Estimating the approximation error as the proof of Theorem \ref{app of anisotropic}, we get the error bound (\ref{temp3}).
\end{proof}

\subsection{Approximation of deep composition model}\label{sec: app of dcm}

In this section, we study neural network approximation of the deep composition model
\[
\cF_{\deep} := \left\{ f_M\circ \cdots \circ f_1\ |\ f_m: [0,1]^{d_m} \to [0,1]^{d_{m+1}}, \|f_{m,j}\|_{\cB^{\Bs_m}_{q_m}([0,1]^{d_m})}\le 1, \forall j\in [1:d_{m+1}] \right\}.
\]
Recall that, by the embedding property of anisotropic Besov space and the definition of $\gamma_m$ in (\ref{gamma_m}), we know that $f_{m,j}$ is $\gamma_m$-H\"older continuous for all $j\in [1:d_{m+1}]$. In other words, 
\[
\|f_m(\Bx) - f_m(\Bx')\|_\infty \lesssim \|\Bx-\Bx'\|_\infty^{\gamma_m},\quad \forall \Bx,\Bx' \in [0,1]^{d_m}.
\]
The proof of Theorem \ref{app of dcm} is similar to \citet[Theorem 1]{suzuki2021deep}.

\begin{proof}[Proof of Theorem \ref{app of dcm}]
Since $\tilde{s}_m >1/q_m$, by Theorem \ref{app of anisotropic}, for each $f_{m,j}$ and sufficiently large $W_0$ and $L_0$, there exits $g_{m,j}\in \NN(W_0,L_0)$ such that 
\[
\| f_{m,j} - g_{m,j}\|_{L^\infty([0,1]^d)} \lesssim (W_0L_0)^{-2 \tilde{s}_m}.
\]
If necessary, we can modify $g_{m,j}$ so that $0\le g_{m,j}(\Bx)\le 1$ by adding one additional clipping layer, since the clipping can be implemented by the network $\sigma(x)-\sigma(x-1) = \min\{\max\{x,0\},1\}$. Let $g_m=(g_{m,1},\dots,g_{m,d_m})$ and $g=g_M \circ \cdots \circ g_1$, then by Proposition \ref{basic constr}, it is easy to see that $g\in \NN(W,L)$ with 
\[
W = \max_{2\le m\le M} d_m W_0,\quad L = ML_0.
\]
The approximation error can be estimated as 
\begin{align*}
&\| f - g\|_{L^\infty([0,1]^d)} \\
=& \| f_M\circ \cdots \circ f_1 - g_M\circ \cdots \circ g_1\|_{L^\infty([0,1]^d)} \\
\le& \sum_{m=1}^M \| f_M\circ \cdots \circ f_{m+1} \circ f_m \circ g_{m-1} \circ \cdots \circ g_1 - f_M\circ \cdots \circ f_{m+1} \circ g_m \circ g_{m-1} \circ \cdots \circ g_1 \|_{L^\infty([0,1]^d)}\\
\le& \sum_{m=1}^M \| f_M\circ \cdots \circ f_{m+1} \circ f_m  - f_M\circ \cdots \circ f_{m+1} \circ g_m \|_{L^\infty([0,1]^{d_m})}.
\end{align*}
Since each $f_m$ is $\gamma_m$-H\"older continuous, the function $f_M\circ \cdots \circ f_{m+1}$ is $\widetilde{\gamma}_m$-H\"older continuous with $\widetilde{\gamma}_m := \prod_{k=m+1}^{M} \gamma_k$. As a consequence, 
\begin{align*}
&\| f_M\circ \cdots \circ f_{m+1} \circ f_m  - f_M\circ \cdots \circ f_{m+1} \circ g_m \|_{L^\infty([0,1]^{d_m})} \\
\le& \|f_m - g_m\|_{L^\infty([0,1]^{d_m})}^{\widetilde{\gamma}_m} \lesssim (W_0L_0)^{-2 \tilde{s}_m \widetilde{\gamma}_m}.
\end{align*}
Therefore,
\[
\| f - g\|_{L^\infty([0,1]^d)} \lesssim \sum_{m=1}^M (W_0L_0)^{-2 \tilde{s}_m \widetilde{\gamma}_m} \lesssim (WL)^{-2 s^*},
\]
where we use  $s^* = \min_{1\le m\le M} \tilde{s}_m \widetilde{\gamma}_m$ by the definition (\ref{smoothness of dcm}).
\end{proof}

\subsection{Approximation lower bounds}

In this section, we are going to derive lower bounds for the approximation error of deep ReLU neural networks on smooth function classes and prove Theorems \ref{anisotropic lower bound}, \ref{mixed lower bound} and \ref{dcm lower bound}. The key concept in our analysis is the notion of pseudo-dimension or VC-dimension \citep{vapnik1971uniform}, which was used in \citep{achour2022general,shen2022optimal,siegel2023optimal,yarotsky2018optimal} to prove lower bounds for neural network approximation of smooth functions, such as H\"older and isotropic Besov functions. We will generalize these results to anisotropic and mixed smooth Besov functions and the deep composition model.

\begin{definition}[Pseudo-dimension]
Let $\cF$ be a class of real-valued functions defined on a set $\cX$. A set $S=\{x_1,\dots,x_m\}\subseteq \cX$ is shattered by $\cF$ if there exist constants $t_1,\dots,t_N\in \bR$ such that
$$
|\{ \sgn(f(x_1)-t_1),\dots,\sgn(f(x_N)-t_N): f\in \cF \}| =2^N.
$$
We say that $t=(t_1,\dots,t_N)$ witnesses that shattering. The pseudo-dimension $\Pdim(\cF)$ is the maximum cardinality of a set $S\subseteq \cX$ that is shattered by $\cF$.
\end{definition}

The pseudo-dimension of neural network classes have been extensively studied \citep{anthony2009neural}. Nearly tight bounds for neural networks with piecewise polynomial activation functions have been established \citep{bartlett2019nearly}. In particular, for ReLU neural network with width $W$ and depth $L$, we have
\begin{equation}\label{Pdim of NN}
\Pdim(\NN(W,L)) \lesssim W^2L^2\log(WL) \land W^3L^2,
\end{equation}
which was proved by \citet[Theorems 7 and 10]{bartlett2019nearly}. Note that the second bound is most informative when the width $W$ is bounded and the depth $L$ is large.

Another concept that we need in our analysis is the metric entropy of a function class, which has been widely used to quantify the complexity of a model in approximation theory and statistical learning theory.

\begin{definition}[Metric entropy]
Let $\epsilon>0$, $\rho$ be a pseudo-metric on $\cV$ and $S \subseteq\cV$. A set $A\subseteq \cV$ is called an $\epsilon$-covering of $S$ if for any $x\in S$ there exists $y\in A$ such that $\rho(x,y)\le \epsilon$. A subset $B\subseteq S$ is an $\epsilon$-packing of $S$ (or $\epsilon$-separated) if any two elements $x\neq y$ in $B$ satisfy $\rho(x,y)>\epsilon$. The $\epsilon$-covering and $\epsilon$-packing numbers of $S$ are defined respectively by
\begin{align*}
\cN(\epsilon, S, \rho) &:= \min\{|A|: A \mbox{ is an $\epsilon$-covering of } S \}, \\
\cM(\epsilon, S, \rho) &:= \max\{|B|: B \mbox{ is an $\epsilon$-packing of } S \}.
\end{align*}
The quantities $\log \cN(\epsilon, S, \rho)$ and $\log \cM(\epsilon, S, \rho)$ are called metric entropy of $S$.
\end{definition}

It is not hard to check that $\cM(2\epsilon, S,\rho) \le \cN(\epsilon, S,\rho) \le \cM(\epsilon, S,\rho)$. The metric entropy determines the grow rates of the covering and packing numbers as $\epsilon \to 0$. In many cases, the metric $\rho$ is induced by a norm $\|\cdot\|$ and we denote $\cN(\epsilon, S, \|\cdot\|)$ and $\cM(\epsilon, S, \|\cdot\|)$ for convenience.

For function classes in $L^p(\mu)$ with $p\ge 1$ and probability measure $\mu$, the metric entropy and pseudo-dimension are related as follows \citep[Theorem 2.6.7]{vaart2023weak}. An envelope function of a real-valued function class $\cF$ defined on $\cX$ is any function $F(x)$ such that $|f(x)|\le F(x)$ for every $x\in \cX$ and $f\in \cF$. If $0<\|F\|_{L^p(\mu)}<\infty$, then
\begin{equation}\label{metric bound by pdim}
\cN(\epsilon,\cF, \|\cdot\|_{L^p(\mu)}) \le C \Pdim(\cF) (16e)^{\Pdim(\cF)} \left( \frac{\|F\|_{L^p(\mu)}}{\epsilon}\right)^{p \Pdim(\cF)},
\end{equation}
for a universal constant $C>0$ and $0<\epsilon<\|F\|_{L^p(\mu)}$. We will apply inequalities (\ref{Pdim of NN}) and (\ref{metric bound by pdim}) to derive approximation lower bounds for neural networks. The main idea is that, if a neural network can approximate a complicated function class, then its metric entropy should be large and hence the size of this network is also large by (\ref{Pdim of NN}) and (\ref{metric bound by pdim}).

\subsubsection{Proof of Theorem \ref{anisotropic lower bound}}

Since $L^p([0,1]^d) \hookrightarrow L^1([0,1]^d)$ and $\cB^\Bs_{\infty,1}([0,1]^d) \hookrightarrow \cB^\Bs_{q,r}([0,1]^d)$ for any $1\le p,q,r\le \infty$, we only need to prove the lower bound for $p=1$, $q=\infty$ and $r=1$. Notice that it is enough to consider the approximation on $\Omega=[0,1)^d$ as before. Our proof is base on the idea from \citet{maiorov1999degree}.

Let us first construct a series of sub-classes of $\cF=\{f:\|f\|_{\cB^\Bs_{\infty,1}(\Omega)} \le 1\}$ that are well separated so that they are difficult to be approximated by neural networks. To do this, we choose a $C^\infty$ bump function $\phi:\bR \to [0,1]$ whose support is strictly contained in $[0,1]$. As in Section \ref{sec: app of anisotropic}, we partition $\Omega$ into disjoint rectangles by (\ref{partition}) with $b=2$ and
\[
\ell_j = \lfloor \ell \tilde{s}/s_j \rfloor, \quad \ell \in \bN.
\]
For any $\ell\in \bN$, let $\cA_\ell :=\{\Ba=(a_\Bi)_{\Bi\in I_\ell}: a_\Bi\in \{0,1\} \}$ be the set of all binary sequences on $I_\ell$. We consider the function parameterized by $\Ba \in \cA_\ell$ as follows
\begin{equation}\label{f_a}
f_\Ba(\Bx) = c_{\phi} 2^{-\tilde{s} \ell} \sum_{\Bi \in I_\ell} a_\Bi \prod_{j=1}^d \phi\left(2^{\ell_j}x_j-i_j \right),
\end{equation}
where $c_\phi$ is a constant independent of $\ell$ that will be chosen later. Notice that the function $\prod_{j=1}^d \phi(2^{-\ell_j}x_j-i_j ) =0$ when $\Bx \notin \Omega_{\ell,\Bi}$. By direct calculation, one can show that 
\[
|f_\Ba|_{\cB^{\Bs,j}_{\infty,1}(\Omega)} \le C c_{\phi} 2^{-\tilde{s} \ell} 2^{s_j \ell_j} |\phi|_{\cB^{s_j}_{\infty,1}([0,1])} \le C c_{\phi} |\phi|_{\cB^{s_j}_{\infty,1}([0,1])},
\]
for some constant $C>0$ depending on $\Bs$ and $d$. Thus, we can choose $c_\phi$ such that $\|f_\Ba\|_{\cB^\Bs_\infty(\Omega)} \le 1$ so that $f_\Ba \in \cF$ for each $\Ba \in \cA_\ell$. Let us denote $N=|I_\ell| \asymp 2^\ell$, then $|\Omega_{\ell,\Bi}|=|I_\ell|^{-1}=N^{-1}$ for any $\Bi \in I_\ell$. If $N \ge 8$, then there exists a subset $\cB_\ell\subseteq \cA_\ell$ whose cardinality $|\cB_\ell|\ge 2^{N/4}$, such that any two sign vectors $\Ba\neq \Ba'$ in $\cB_\ell$ are different in more than $\lfloor N/8 \rfloor$ places (see \citet[Lemma 3.8]{jiao2023approximation} or Varshamov-Gilbert bound in \citet[Lemma 2.9]{tsybakov2009introduction}). Consequently, for $\Ba\neq \Ba'$ in $\cB_\ell$, 
\begin{align*}
\|f_\Ba - f_{\Ba'} \|_{L^1(\Omega)} &= c_{\phi} 2^{-\tilde{s} \ell} \sum_{\Bi:a_\Bi\neq a_\Bi'} \int_{\Omega_{\ell,\Bi}} |a_\Bi-a_\Bi'| \prod_{j=1}^d \phi\left(2^{\ell_j}x_j-i_j \right) d\Bx \\
&= c_{\phi} 2^{-\tilde{s} \ell} \sum_{\Bi:a_\Bi\neq a_\Bi'} N^{-1} \|\phi\|_{L^1([0,1])}^d \\
&\ge c_{\phi} 2^{-\tilde{s} \ell} \|\phi\|_{L^1([0,1])}^d/8 = 3c_1 2^{-\tilde{s} \ell},
\end{align*}
where $c_1=c_{\phi} \|\phi\|_{L^1([0,1])}^d/24$.

Let $\epsilon>0$ be arbitrarily small and denote the approximation error by
\[
\delta = \sup_{\|f\|_{\cB^\Bs_{\infty,1}(\Omega)} \le 1} \inf_{g\in \NN(W,L)} \|f-g\|_{L^1(\Omega)} + \epsilon.
\]
Then, for any $\Ba\in \cB_\ell$, there exists $g_\Ba \in \NN(W,L)$ such that $\|f_\Ba - g_\Ba\|_{L^1(\Omega)} \le \delta$. Let $\cT_B$ be the truncation operator defined by (\ref{trancation}) with $B= c_{\phi} 2^{-\tilde{s} \ell}$. Since $|f_\Ba(\Bx)|\le B$ for any $\Ba \in \cB_\ell$, we have
\[
\|f_\Ba - \cT_B g_\Ba\|_{L^1(\Omega)} \le \|f_\Ba - g_\Ba\|_{L^1(\Omega)} \le \delta.
\]
Thus, for any $\Ba\neq \Ba'$ in $\cB_\ell$,
\begin{align*}
\|\cT_B g_\Ba - \cT_B g_{\Ba'}\|_{L^1(\Omega)} &\ge \|f_\Ba - f_{\Ba'} \|_{L^1(\Omega)} - \|f_\Ba - \cT_B g_\Ba\|_{L^1(\Omega)} - \|f_{\Ba'} - \cT_B g_{\Ba'}\|_{L^1(\Omega)} \\
&\ge 3c_1 2^{-\tilde{s} \ell} -2\delta.
\end{align*}

Next, we show that $\delta \ge c_1 2^{-\tilde{s} \ell}$ for some $\ell$ depending on the width $W$ and depth $L$. Suppose on the contrary that $\delta < c_1 2^{-\tilde{s} \ell}$, then $\|\cT_B g_\Ba - \cT_B g_{\Ba'}\|_{L^1(\Omega)} > c_1 2^{-\tilde{s} \ell}$ for $\Ba\neq \Ba'$ in $\cB_\ell$, which implies that the following lower bound for the packing number of $\cT_B \NN(W,L)$,
\begin{equation}\label{temp4}
\cM(c_1 2^{-\tilde{s} \ell},\cT_B \NN(W,L), \|\cdot\|_{L^1(\Omega)}) \ge |\cB_\ell| \ge 2^{N/4}.
\end{equation}
On the other hand, let $P = \Pdim(\cT_B \NN(W,L))$ be the pseudo-dimension of the truncated neural network. Since the truncation can be implemented by the ReLU neural network $\sigma(t) - \sigma(-t) -\sigma(t-B) +\sigma(-t-B) \in \NN(4,1)$. Thus, the truncation of a neural network is also a neural network and $\cT_B \NN(W,L) \subseteq \NN(\max\{W,4\},L+1)$ by Proposition \ref{basic constr}. By inequality (\ref{Pdim of NN}), we have
\begin{equation}\label{temp5}
P \lesssim W^2L^2\log(WL) \land W^3L^2.
\end{equation}
Using inequality (\ref{metric bound by pdim}) and $B= c_{\phi} 2^{-\tilde{s} \ell}$, we get
\begin{align*}
\cM(c_1 2^{-\tilde{s} \ell},\cT_B \NN(W,L), \|\cdot\|_{L^1(\Omega)}) &\le \cN(c_1 2^{-\tilde{s} \ell}/2,\cT_B \NN(W,L), \|\cdot\|_{L^1(\Omega)}) \\
&\le C P (16e)^P \left( \frac{2B}{c_1 2^{-\tilde{s}\ell}}\right)^P \le c_2^P,
\end{align*}
for some constant $c_2>0$. Combining this bound with (\ref{temp4}), we conclude that
\[
P \ge (4 \log_2 c_2)^{-1} N \ge c_3 2^{\ell},
\]
for some constant $c_3>0$. If we choose $\ell = \lceil \log_2 (P/c_3) \rceil +1$, then the above inequality must be false so that we get a contradiction for the assumption $\delta < c_1 2^{-\tilde{s} \ell}$. Hence, for this choice of $\ell$, we have
\[
\delta \ge c_1 2^{-\tilde{s} \ell} \gtrsim P^{-\tilde{s}} \gtrsim (W^2L^2\log(WL) \land W^3L^2)^{-\tilde{s}}.
\]
By the definition of $\delta$, we finish the proof.

\subsubsection{Proof of Theorem \ref{mixed lower bound}}

As in the proof of Theorem \ref{anisotropic lower bound}, we can assume that $p=1$ and $q=\infty$. We will use similar argument as Theorem \ref{anisotropic lower bound}. The main difference is that we are going to apply the following estimate of the entropy of the mixed Besov space $\cF=\{f:\|f\|_{\MB^s_{\infty,r}(\Omega)} \le 1\}$  from \citet{vybiral2006function} and \citet[Page 103]{dung2018hyperbolic}. For any $n\in \bN$, if $\epsilon_n >0$ satisfies $\cM(\epsilon_n, \cF, L^1(\Omega)) \le 2^n$, then
\begin{equation}\label{temp6}
\epsilon_n \gtrsim n^{-s} (\log n)^{(d-1)(s+1/2-1/r)}.
\end{equation}
We choose an $\epsilon_n$-packing $\cF_n$ of $\cF$ with $|\cF_n|=2^n$.

Let $\epsilon>0$ be arbitrarily small and denote the approximation error by
\[
\delta = \sup_{\|f\|_{\MB^s_\infty(\Omega)} \le 1} \inf_{g\in \NN(W,L)} \|f-g\|_{L^1(\Omega)} + \epsilon.
\]
Then, for any $f\in \cF$, there exists $g\in \NN(W,L)$ such that $\|f-g\|_{L^1(\Omega)} \le \delta$. Since $|f(\Bx)|\le 1$, we also have 
\[
\|f - \cT_1 g\|_{L^1(\Omega)} \le \|f - g\|_{L^1(\Omega)} \le \delta.
\]
Thus, for any $f \neq f'$ in $\cF_n$, there exists $g,g'\in \NN(W,L)$ such that
\[
\|\cT_1 g - \cT_1 g'\|_{L^1(\Omega)} \ge \|f - f' \|_{L^1(\Omega)} - \|f - \cT_1 g\|_{L^1(\Omega)} - \|f' - \cT_1 g'\|_{L^1(\Omega)} \ge \epsilon_n -2\delta.
\]

Suppose $\delta<\epsilon_n/3$, then $\|\cT_1 g - \cT_1 g'\|_{L^1(\Omega)} \ge \epsilon_n/3$, which implies
\[
\cM(\epsilon_n/3,\cT_1 \NN(W,L), \|\cdot\|_{L^1(\Omega)}) \ge |\cF_n| = 2^n.
\]
If we denote the pseudo-dimension by $P = \Pdim(\cT_B \NN(W,L))$, then we have the estimate (\ref{temp5}) and, by inequality (\ref{metric bound by pdim}), 
\[
\cM(\epsilon_n/3,\cT_1 \NN(W,L), \|\cdot\|_{L^1(\Omega)}) \le C P (16e)^P \left( \frac{6}{\epsilon_n}\right)^P \le \left( \frac{c_1}{\epsilon_n}\right)^P,
\]
for some constant $c_1>0$. As a consequence, we get
\[
P \ge (\log_2 (c_1/\epsilon_n))^{-1} n \ge c_2 \frac{n}{\log_2 n},
\]
for some constant $c_2>0$, where we use the lower bound (\ref{temp6}) in the second inequality. It is possible to choose 
\[
n\asymp P\log P \lesssim (W^2L^2\log(WL) \land W^3L^2) \log(WL),
\]
such that the above inequality is always false and hence we get a contradiction for the assumption $\delta<\epsilon_n/3$. With this choice of $n$,
\begin{align*}
\delta &\ge \epsilon_n/3 \gtrsim n^{-s} (\log n)^{(d-1)(s+1/2-1/r)} \\
&\gtrsim (W^2 L^2\log (WL) \land W^3L^2)^{-s} (\log (WL))^{(d-1)(s+1/2-1/r)-s},
\end{align*}
which finishes the proof by the definition of $\delta$.

\subsubsection{Proof of Theorem \ref{dcm lower bound}}

As in the proof of Theorem \ref{anisotropic lower bound}, we first construct a series of sub-classes of $\cF_{\deep}$ that are well separated and it is enough to consider the error in $L^1([0,1]^d)$ norm. The construction follows similar strategy developed by \citet[Proof of Theorem 3]{schmidthieber2020nonparametric} and \citet[Proof of Theorem 4]{suzuki2021deep}. Recall that the smoothness index is defined as
\[
s^{**} := \min_{1\le m\le M} \tilde{s}_m \prod_{k=m+1}^{M} \widetilde{\gamma}_k, \quad \widetilde{\gamma}_k:= (\underline{s}_k-1/q_k) \land 1.
\]
We let $m^*= \argmin_{1\le m\le M} \tilde{s}_m \prod_{k=m+1}^{M} \widetilde{\gamma}_k$ and denote $\gamma^*:=  \prod_{k=m^*+1}^{M} \widetilde{\gamma}_k$, so that $s^{**}= \tilde{s}_{m^*} \gamma^*$. Without loss of generality, we may assume that the anisotropic smoothness vector $\Bs_m =(s_{m,1},\dots,s_{m,d_m})$ satisfies $\underline{s}_m  = \min_{1\le j\le d_m} s_{m,j} = s_{m,1}$ for $m\in [1:M]$.

For any $\ell\in \bN$ and $\Ba \in \cA_\ell =\{(a_\Bi)_{\Bi\in I_\ell}: a_\Bi\in \{0,1\} \}$, we let $f_\Ba$ be the function defined by (\ref{f_a}) in the proof of Theorem \ref{anisotropic lower bound} with $\Bs = \Bs_{m^*}$ and $d=d_{m^*}$. Then, by choosing $c_\phi$, we are guaranteed that  $\|f_\Ba\|_{\cB^{\Bs_{m^*}}_\infty([0,1]^{d_{m^*}})} \le 1$ for each $\Ba \in \cA_\ell$. We define $h_\Ba = f_M\circ \cdots \circ f_1$, where
\begin{align*}
f_m(x_1,\dots,x_{d_m}) &= (x_1,\dots,x_{d_{m+1}})^\top,\quad m=1,\dots,m^*-1, \\
f_{m^*}(x_1,\dots,x_{d_{m^*}}) &= (f_\Ba(x_1,\dots,x_{d_{m^*}}),0,\dots,0)^\top, \\
f_m(x_1,\dots,x_{d_m}) &= (c_m x_1^{\widetilde{\gamma}_m},0,\dots,0)^\top, \quad m=m^*+1,\dots, M.
\end{align*}
Note that the above definition is well-defined for $m=1,\dots,m^*-1$ due to the assumption $d_1\ge d_2\ge \cdots \ge d_M$. For $m=m^*+1,\dots, M$, through a cumbersome calculation, one can verify that $x_1^{\widetilde{\gamma}_m} \in \cB^{\underline{s}_m}_{q_m}([0,1])$. Hence, we can ensure that $\|f_{m,1}\|_{\cB^{\Bs_m}_{q_m}([0,1]^{d_m})}\le 1$ by choosing sufficiently small constant $c_m>0$. As a consequence, we have $h_\Ba\in \cF_{\deep}$. Observe that $h_\Ba$ can be computed explicitly as
\begin{align*}
h_\Ba(\Bx) &= c_0 \left(2^{-\tilde{s}_{m^*} \ell} \sum_{\Bi \in I_\ell} a_\Bi\prod_{j=1}^d \phi\left(2^{\ell_j}x_j-i_j \right)\right)^{\prod_{k=m^*+1}^{M} \widetilde{\gamma}_k} \\
&= c_0 2^{-s^{**}\ell} \sum_{\Bi \in I_\ell} a_\Bi\prod_{j=1}^d \phi^{\gamma^*}\left(2^{\ell_j}x_j-i_j \right),
\end{align*}
for some constant $c_0$ depending on $\phi$ and $(\Bs_m,q_m,d_m)_{m=1}^M$. As in the proof of Theorem \ref{anisotropic lower bound}, we denote $N=|I_\ell| \asymp 2^\ell$ and choose a subset $\cB_\ell\subseteq \cA_\ell$ with $|\cB_\ell|\ge 2^{N/4}$ such that any two sign vectors $\Ba\neq \Ba'$ in $\cB_\ell$ are different in more than $\lfloor N/8 \rfloor$ places. Then, for any $\Ba\neq \Ba'$ in $\cB_\ell$, 
\begin{align*}
\|h_\Ba - h_{\Ba'} \|_{L^1([0,1]^d)} &= c_0 2^{-s^{**} \ell} \sum_{\Bi:a_\Bi\neq a_\Bi'} \int_{\Omega_{\ell,\Bi}} |a_\Bi-a_\Bi'| \prod_{j=1}^d \phi^{\gamma^*} \left(2^{\ell_j}x_j-i_j \right) d\Bx \\
&= c_0 2^{-s^{**} \ell} \sum_{\Bi:a_\Bi\neq a_\Bi'} N^{-1} \|\phi^{\gamma^*}\|_{L^1([0,1])}^d \\
&\ge c_0 2^{-s^{**} \ell} \|\phi^{\gamma^*}\|_{L^1([0,1])}^d/8 = 3c_1 2^{-s^{**} \ell},
\end{align*}
where $c_1=c_0 \|\phi^{\gamma^*}\|_{L^1([0,1])}^d/24$. Thus, we have constructed a subset of $\cF_{\deep}$, which is difficult to approximate. The remaining proof can be processed as Theorem \ref{anisotropic lower bound}, where the smoothness index $\tilde{s}$ is replaced by $s^{**}$.

\section{Proofs of learning rates}\label{sec: proofs of learning}

In this section, we give the proofs of Theorems \ref{learn anisotropic}, \ref{learn mixed} and \ref{learn dcm}. Our proofs are based on the following lemma from \citet[Supplement B, Lemma 18]{kohler2021rate}.

\begin{lemma}\label{error decomp}
Assume that the data distribution satisfies (\ref{noise condition}) and the regression function $f$ is bounded in absolute value. Let $\widehat{h}_n$ be the least squares estimator (\ref{least squares}) based on some function class $\cH$. If $B_n = c\log n$ for some constant $c>0$, then 
\[
\bE_{\cD_n} \|\cT_{B_n}\widehat{h}_n-f\|_{L^2(\mu)}^2 \le C \cE_{gen} + 2 \cE_{app},
\]
where $C$ is a constant independent of the regression function $f$ and the sample size $n$ and
\begin{align*}
\cE_{gen} :=& \frac{(\log n)^2 \sup_{X_{1:n}\in ([0,1]^d)^n}\log (\cN(n^{-1}B_n^{-1}, \cT_{B_n}\cH,\|\cdot\|_{L^1(X_{1:n})})+1)}{n}, \\
\cE_{app} :=& \inf_{h\in \cH} \|f-h\|_{L^2(\mu)}^2.
\end{align*}
Here, $X_{1:n} =(X_1,\dots,X_n)$ denotes the sequence of sample points and $\cN(\epsilon, \cT_{B_n}\cH,\|\cdot\|_{L^1(X_{1:n})})$ denotes the $\epsilon$-covering number of the function class $\cT_{B_n}\cH$ in the metric $\|h_1-h_2\|_{L^1(X_{1:n})} = \frac{1}{n}\sum_{i=1}^n|h_1(X_i)-h_2(X_i)|$. 
\end{lemma}

Note that Lemma \ref{error decomp} decomposes the excess risk into the generalization gap $\cE_{gen}$ and the approximation error $\cE_{app}$. For the generalization gap, when $\cH=\NN(W_n,L_n)$ is a ReLU neural network, the truncation $\cT_B \NN(W_n,L_n) \subseteq \NN(\max\{W_n,4\},L_n+1)$ is also a network (see the proof above inequality (\ref{temp5})). Hence, by inequalities (\ref{Pdim of NN}) and (\ref{metric bound by pdim}),
\begin{align*}
\log (\cN(n^{-1}B_n^{-1}, \cT_{B_n}\cH,\|\cdot\|_{L^1(X_{1:n})})+1) &\lesssim \Pdim(\cT_{B_n} \cH) \log(nB_n^2) \\
&\lesssim W_n^2L_n^2 \log(W_nL_n) \log(nB_n^2),
\end{align*}
which implies, for $B_n = c\log n$, 
\[
\cE_{gen} \lesssim \frac{W_n^2L_n^2 \log(W_nL_n) (\log n)^3}{n}.
\]
We can apply our approximation results to estimate $\cE_{app}$ for different smooth function class $f\in \cF$. The trade-off between $\cE_{gen}$ and $\cE_{app}$ tells us how to choose the size of the network and gives the convergence rate for the estimator.

\begin{proof}[Proof of Theorem \ref{learn anisotropic}]
By Theorem \ref{app of anisotropic}, since $\tilde{s}>1/q$,
\[
\cE_{app} \le \inf_{h\in \NN(W_n,L_n)} \|f-h\|_{L^\infty([0,1]^d)}^2 \lesssim (W_nL_n)^{-4\tilde{s}}.
\]
If we choose $W_nL_n \asymp n^{\frac{1}{4\tilde{s}+2}} (\log n)^{-\frac{2}{2\tilde{s}+1}}$, then 
\begin{align*}
\cE_{app} &\lesssim (W_nL_n)^{-4\tilde{s}} \lesssim n^{-\frac{2\tilde{s}}{2\tilde{s}+1}} (\log n)^{\frac{8\tilde{s}}{2\tilde{s}+1}},\\
\cE_{gen} &\lesssim \frac{W_n^2L_n^2 \log(W_nL_n) (\log n)^3}{n} \lesssim n^{-\frac{2\tilde{s}}{2\tilde{s}+1}} (\log n)^{\frac{8\tilde{s}}{2\tilde{s}+1}}.
\end{align*}
Lemma \ref{error decomp} gives the desired convergence rate.
\end{proof}

\begin{proof}[Proof of Theorem \ref{learn mixed}]
By Theorem \ref{app of mixed}, since $s>1/q$,
\[
\cE_{app} \le \inf_{h\in \NN(W_n,L_n)} \|f-h\|_{L^\infty([0,1]^d)}^2 \lesssim (W_nL_n)^{-4s} (\log W \log L)^{2(d-1)(2s+1)}.
\]
If we choose $W_nL_n \asymp n^{\frac{1}{4s+2}} (\log n)^{-\frac{2}{2s+1}+2d-2}$, then 
\begin{align*}
\cE_{app} &\lesssim (W_nL_n)^{-4s} (\log W \log L)^{2(d-1)(2s+1)} \lesssim n^{-\frac{2s}{2s+1}} (\log n)^{\frac{8s}{2s+1}+4d-4},\\
\cE_{gen} &\lesssim \frac{W_n^2L_n^2 \log(W_nL_n) (\log n)^3}{n} \lesssim n^{-\frac{2s}{2s+1}} (\log n)^{\frac{8s}{2s+1}+4d-4}.
\end{align*}
Lemma \ref{error decomp} gives the desired convergence rate.
\end{proof}

\begin{proof}[Proof of Theorem \ref{learn dcm}]
By Theorem \ref{app of dcm}, 
\[
\cE_{app} \le \inf_{h\in \NN(W_n,L_n)} \|f-h\|_{L^\infty([0,1]^d)}^2 \lesssim (WL)^{-4s^*}.
\]
If we choose $W_nL_n \asymp n^{\frac{1}{4s^*+2}} (\log n)^{-\frac{2}{2s^*+1}}$, then 
\begin{align*}
\cE_{app} &\lesssim (W_nL_n)^{-4s^*} \lesssim n^{-\frac{2s^*}{2s^*+1}} (\log n)^{\frac{8s^*}{2s^*+1}},\\
\cE_{gen} &\lesssim \frac{W_n^2L_n^2 \log(W_nL_n) (\log n)^3}{n} \lesssim n^{-\frac{2s^*}{2s^*+1}} (\log n)^{\frac{8s^*}{2s^*+1}}.
\end{align*}
Lemma \ref{error decomp} gives the desired convergence rate.
\end{proof}

\section*{Acknowledgments}

The work of Y. Yang was partially supported by National Natural Science Foundation of China under Grants 12501131 and 12526216. The work by J. Fan was partially supported by the Research Grants Council of Hong Kong [Project No. HKBU12302923 and HKBU12303024], and Guangdong and Hong Kong Universities ``1+1+1'' Joint Research Collaboration Scheme 2025A0505000007.

\appendix

\section{Approximation by convolutional neural networks}\label{appendix: cnn}

In this appendix, we show that one can derive approximation results for convolutional neural networks from our approximation bounds for fully-connected networks by using the idea that fully-connected networks can be represented by CNNs with comparable size, which has been widely used in the analysis of CNNs  \citep{zhou2020theory,zhou2020universality,oono2019approximation,mao2022approximation,yang2025rates,li2025higher}. Let us first recall the architecture of downsampled CNN introduced by \citet{zhou2020theory}. 

Let $\Bw=(w_0,w_1,\dots,w_s)^\top \in \bR^{s+1}$ be a filter vector with filter length $s\ge 1$. The discrete convolution of the filter $\Bw$ with an input vector $\Bx=(x_1,\dots,x_d)^\top\in \bR^d$ is a vector $\Bw * \Bx \in \bR^{d+s}$, whose $i$-th component is given by
\[
(\Bw * \Bx)_i = \sum_{k=1}^d w_{i-k}x_k,
\]
where $w_{i-k}=0$ when $i-k \notin \{0,1,\dots,s\}$. If we view $\Bw * \Bx$ as a linear transform of $\Bx$, then it corresponds to a Toeplitz type convolutional matrix
\[
T_\Bw := 
\begin{pmatrix}
w_0 & 0 & 0 & 0 & \cdots & \cdots & 0 \\
w_1 & w_0 & 0 & 0 & \cdots & \cdots & 0 \\
\vdots & \ddots & \ddots & \ddots & \ddots & \ddots & \vdots \\
w_s & w_{s-1} & \cdots & w_0 & 0 & \cdots & 0 \\
0 & w_s & \cdots & w_1 & \ddots & \ddots & 0 \\
\vdots & \ddots & \ddots & \ddots & \ddots & \ddots & \vdots \\
0 & \cdots & 0 & w_s & \cdots & w_1 & w_0 \\
0 &\cdots & \cdots & 0 & w_s & \cdots & w_1 \\
\vdots & \ddots & \ddots & \ddots & \ddots & \ddots & \vdots \\
0 & 0 & 0 & 0 & \cdots & 0 & w_0
\end{pmatrix}\in \bR^{(d+s)\times d}.
\]
The downsampling operator $\cD_m:\bR^{D} \to \bR^{\lfloor D/m \rfloor}$ with a scaling parameter $m\le D$ is defined by 
\[
\cD_m(\Bv) = (v_{im})_{i=1}^{\lfloor D/m \rfloor},\quad \Bv \in \bR^D.
\] 
Let $J\in \bN$ denote the depth and $\cJ=\{J_k\}_{k=1}^L$ with $1<J_1\le J_2\le \dots \le J_L=J$ be the layers where we introduce $L$ downsamplings. A downsampled CNN with filter length $s$ and downsamplings at layers $\cJ$ has widths $\{d_j\}_{j=0}^J$ defined iteratively by $d_0$ and for $k=1,\dots,L$,
\[
d_j = 
\begin{cases}
d_{j-1}+s, &\mbox{if }J_{k-1}<j<J_k, \\
\lfloor (d_{j-1}+s)/d_{J_{k-1}} \rfloor, &\mbox{if }j=J_k.
\end{cases}
\]
In the hidden layers, the downsampled CNN computes a sequence of vector-valued functions $h_j(\Bx):\bR^d \to \bR^{d_j}$, $j=1,\dots,J$, defined iteratively by $h_0(\Bx)=\Bx$ and for $k=1,\dots,L$,
\[
h_j(\Bx) = 
\begin{cases}
\sigma(T_{\Bw_j}h_{j-1}(\Bx) + \Bb_j ), &\mbox{if }J_{k-1}<j<J_k, \\
\cD_{d_{J_{k-1}}}(\sigma(T_{\Bw_j}h_{j-1}(\Bx) + \Bb_j )), &\mbox{if }j=J_k.
\end{cases}
\]
Moreover, in order to reduce the number of free parameters, we require that the bias vectors $\Bb_j\in \bR^{d_{j-1}+s}$ satisfy $b_{j,s+1} = b_{j,s+2}= \cdots = b_{j,d_{j-1}}$ for $j\in \cJ$. The output of the downsampled CNN is defined by 
\[
h(\Bx) = \Ba^\top h_J(\Bx) + b, \quad \Ba\in \bR^{d_J}, b\in \bR.
\]
For convenience, we denote the set of functions $h$ that can be parameterized by the downsampled CNN as above by $\CNN(s,J,\cJ)$. The following theorem from \citet[Theorem 2]{zhou2020theory} helps us translate approximation results for fully-connected networks to CNNs.

\begin{theorem}\label{CNN}
For any $g\in \NN_{d,1}(W,L)$ with $W\ge d$ and $s\in [2:W^2]$, there exists $h\in \CNN(s,J,\cJ)$ with $L$ downsamplings at layers $\cJ=\{J_k=\sum_{j=1}^k \Delta_j: k=1,\dots,L\}$ with $\Delta_j \le \lceil \frac{W^2-1}{s-1} \rceil$ such that $h(\Bx)=g(\Bx)$ for $\Bx\in [0,1]^d$. Furthermore, the number of free parameters in the above CNN is at most 8 times of that of the fully-connected network.
\end{theorem}

To apply Theorem \ref{CNN}, we choose $W$ sufficiently large so that our approximation bounds in Theorems \ref{app of anisotropic}, \ref{app of mixed} and \ref{app of dcm} hold. Then, the CNN in Theorem \ref{CNN} has bounded width, depth $J\asymp L$ and $N\asymp L$ free parameters. Since any functions in $\NN(W,L)$ can be represented by such CNN, we get the following approximation results:

(1) Let $\Bs\in (0,\infty)^d$ and $1\le p,q\le \infty$. If $\tilde{s}>1/q-1/p$, then for any $f\in \cB^\Bs_q([0,1]^d)$ with $\|f\|_{\cB^\Bs_q([0,1]^d)}\le 1$, there exists $h\in \CNN(s,J,\cJ)$ such that  
\[
\|f-h\|_{L^p([0,1]^d)} \lesssim J^{-2\tilde{s}}.
\]

(2) Let $s>0$ and $1\le p,q\le \infty$. If $s>1/q-1/p$, then for any $f\in \MB^s_q([0,1]^d)$ with $ \|f\|_{\MB^s_q([0,1]^d)}\le 1$, there exists $h\in \CNN(s,J,\cJ)$ such that
\[
\|f-h\|_{L^p([0,1]^d)} \lesssim J^{-2s}(\log J)^{(d-1)(2s+1)}.
\]

(3) If $\tilde{s}_m >1/q_m$ for $m\in [1:M]$ in the deep composition model, then for any $f\in \cF_{\deep}$, there exists $h\in \CNN(s,J,\cJ)$ such that
\[
\|f-h\|_{L^\infty([0,1]^d)} \lesssim J^{-2s^*}.
\]

Finally, we remark that one can also derive nearly optimal learning rates for CNNs using arguments similar to Section \ref{sec: learning}. We leave the details to the reader.

\bibliographystyle{myplainnat}
\bibliography{references}
\end{document}